\PassOptionsToPackage{table}{xcolor}
\documentclass[11pt]{article}
\ifdefined\pdfsuppressptexinfo\pdfsuppressptexinfo=-1\fi

\usepackage[preprint]{acl}

\usepackage{times}
\usepackage{latexsym}
\usepackage[T1]{fontenc}
\usepackage[utf8]{inputenc}
\usepackage{microtype}
\IfFileExists{inconsolata.sty}{\usepackage{inconsolata}}{}  %
\usepackage{graphicx}
\graphicspath{{figures/}{../figures/naacl/}}
\usepackage{amsmath,amssymb}
\usepackage{booktabs}
\usepackage{placeins}  %
\usepackage{etoc}  %
\usepackage{multirow}
\usepackage{xcolor}
\usepackage[capitalize]{cleveref}

\newcommand{\lam}{\lambda}
\newcommand{\polS}{\pi_{\theta}}
\newcommand{\polT}{\pi_{T}}
\newcommand{\polB}{\pi_{B}}

\crefname{table}{Table}{Tables}
\crefname{figure}{Figure}{Figures}
\crefname{equation}{Eq.}{Eqs.}
\crefname{section}{\S}{\S\S}
\Crefname{section}{Section}{Sections}
\crefname{appendix}{Appendix}{Appendices}

\definecolor{opdcol}{HTML}{0072B2}   %
\definecolor{fklcol}{HTML}{D55E00}   %
\definecolor{seqcol}{HTML}{009E73}   %
\definecolor{lawcol}{HTML}{8C8C8C}   %
\definecolor{grprow}{gray}{0.93}     %
\definecolor{hitcol}{HTML}{CC79A7}   %
\newcommand{\hit}{}
\newcommand{\hitin}[1]{#1}

\newcommand{\grp}[2]{\addlinespace[2pt]\rowcolor{grprow}\multicolumn{#1}{@{}l@{}}{\textit{#2}}\\}
\newcommand{\hopd}[1]{{\color{opdcol}#1}}
\newcommand{\hfkl}[1]{{\color{fklcol}#1}}
\newcommand{\hseq}[1]{{\color{seqcol}#1}}
\newcommand{\bopd}{\cellcolor{opdcol!10}}
\newcommand{\bfkl}{\cellcolor{fklcol!10}}
\newcommand{\bseq}{\cellcolor{seqcol!10}}
\usepackage{threeparttable}
\newcommand{\tabnote}[1]{\begin{tablenotes}[flushleft]\footnotesize\item[]#1\end{tablenotes}}
\usepackage{longtable}

\title{ErrOPD: When Does On-Policy Distillation Transfer Teacher Errors?}

\author{Xin Li \quad Hao Jiang \quad Annan Wang \quad Yichi Zhang \quad Chau Yuen \\
  Nanyang Technological University, Singapore \\
  \url{https://lixin.ai/ErrOPD}}

\begin{document}
\maketitle
\etocdepthtag.toc{body}

\begin{abstract}
A distilled student can copy its teacher's mistakes along with its skills.
On-policy distillation (OPD), used in recent post-training recipes, has the teacher grade every token of the student's own outputs. Which teacher errors does it pass on?
We build \mbox{ErrOPD}, a testbed whose teachers are fine-tuned to close a JSON list one item early, an error their students almost never make.
The answer depends on the KL direction: across four model families, forward-KL OPD transmits the error within the training budget (+13--23 points), while reverse-KL OPD at our reference learning rate adds under 0.05 points.
A duplicated ID and, where the decision arises, two SQL errors show the same split.
The reverse-KL update scales with the student's own probability of a token, and an intervention tests this: giving reverse KL forward KL's gradient magnitude at the error decision alone makes it transmit on two model pairs.
Reverse-KL OPD can delay such errors rather than block them: some runs adopt the error with longer training, and Qwen3 students that already make it, even rarely, adopt it within the budget.
Because greedy evaluation on Qwen3 misses the inherited error, distilled students should also be audited on sampled outputs.

\end{abstract}

\section{Introduction}
\label{sec:intro}
\begin{figure}[t]
\centering
\includegraphics[width=\columnwidth]{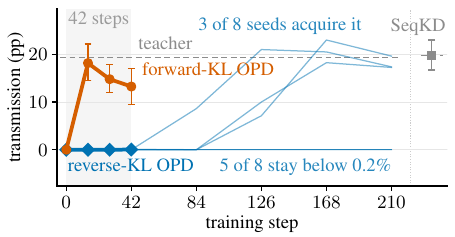}
\caption{\textbf{Forward-KL OPD transmits an injected teacher error within 42 steps and reverse-KL OPD does not, at a learning rate of $10^{-6}$.} y: rise in the student's rate of closing the list early, in points. Dashed: teacher. Qwen3-1.7B student, ESCI.}
\label{fig:one}
\end{figure}

Small language models are increasingly trained to imitate larger ones, a process known as knowledge distillation \citep{hinton2015distilling}.
Classically the small \emph{student} is trained on text the larger \emph{teacher} wrote \citep{kim2016sequence}.
A newer recipe lets the student write its own answers and has the teacher grade every token of them, so that the student is corrected in the situations it actually reaches.
This is on-policy distillation (OPD) \citep{agarwal2024onpolicy,gu2024minillm}.
In the reverse-KL form used by recent post-training recipes, the grade of each sampled token is how much more likely the teacher finds it than the student does \citep{lu2025onpolicydistillation,yang2025qwen3}, and ExOPD adds a coefficient $\lam$ that extrapolates beyond the teacher \citep{yang2026learning}.
We study two forms of on-policy distillation, reverse-KL OPD and forward-KL OPD. Both train on student-generated trajectories but differ in the direction of the token-level distillation objective.

A student that imitates its teacher may copy the teacher's mistakes along with its skills.
A teacher may, for example, end a list one item early or repeat an item in a fraction of its outputs.
Before distilling from such a teacher one needs to know which of its mistakes the student will adopt, in particular mistakes it does not make on its own.

\begin{figure*}[t]
\centering
\includegraphics[width=\textwidth]{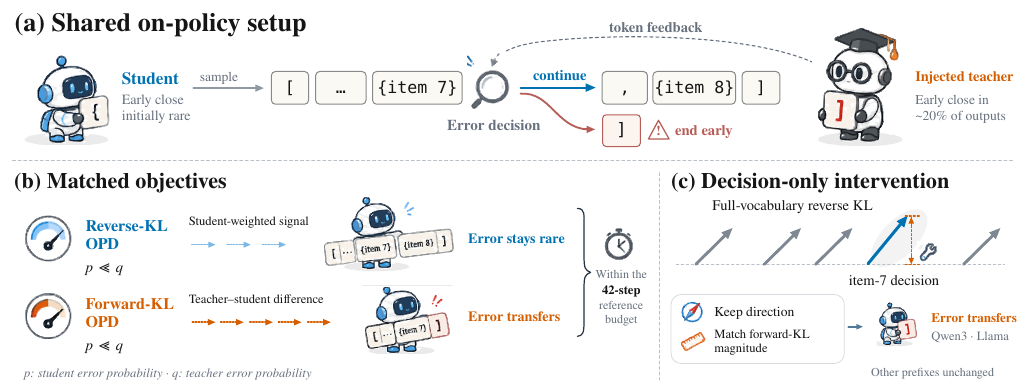}
\caption{\textbf{Both objectives see the same teacher signal at an error decision the student rarely takes, yet only forward KL transmits the error within 42 steps, and reverse KL does once given forward KL's magnitude there (c).}}
\label{fig:mechanism}
\end{figure*}

The usual answer is that reverse KL is mode-seeking and does not penalise a student for missing a mode of the teacher \citep{minka2005divergence,gu2024minillm}.
That is a statement about where an objective is minimised, and here both objectives are minimised at the teacher, which students that already make the error on 1--4\% of samples reach within our budget (\cref{sec:startrate}).
What differs is the speed.
The local reverse-KL logit update contains a student-probability factor: in expectation it changes the student's score for a token in proportion to the probability the student already gives it \citep{mei2020global,armandpour2026unmasking,jiang2026trajectory}, so a decision the student rarely takes receives little gradient however often the teacher takes it, while forward KL, training the student towards the teacher's own choices at the same points, has no such factor (\cref{fig:mechanism}).
This is a local account: it accounts for the direction and the end point of our results but not for their speed (\cref{sec:delay}). Concurrent work derives the factor and observes its outcome for a language style \citep{piskorz2026onpolicy}, but it has not been tested by intervening on the gradient at one decision with an error of known rate.

We test it in a controlled error-injection testbed, which we call ErrOPD.
We fine-tune teachers in four model families to close a JSON list one item early, an error their students almost never make, and distil each student with reverse-KL and with forward-KL OPD, everything else fixed (\cref{sec:setup}).
The answer depends on the KL direction.
Within 42 steps at our reference learning rate, forward KL raised the students' error rate by 13 to 23 points on average, while plain reverse-KL OPD ($\lam=1$) raised it by under 0.05 points in every family and size and stayed below forward KL in every seed pair (\cref{fig:one}).
Reverse KL did train, and forward KL shows that the students can learn the error and that the teacher's signal reaches the error decision (\cref{sec:contrast}).
A duplicated ID in two families and two SQL errors show the same split (\cref{sec:contrast}).
Neither a larger whole-model step nor a full-vocabulary gradient, which samples nothing, accounts for the difference. An intervention does: giving reverse KL forward KL's gradient magnitude at the error decision alone, with every direction unchanged, transmits the error on two model pairs (\cref{sec:alternatives}).

Reverse-KL OPD can delay the transfer of an error initially rare in the student rather than block it: with Qwen3's task, over 210 steps some Qwen3 and Olmo~3 seeds acquired it, and fewer Llama seeds did (\cref{sec:delay}).
The delay concerns injected errors at single decisions: at a larger math budget, students adopted natural teacher phrases they had not produced, whatever their starting rarity (\cref{app:natinh}).
Within the budget reverse-KL OPD does transmit an error the student already makes, even on 0.1--0.3\% of its samples, or one coupled to a decision it takes often, and extrapolation amplifies what has transmitted (\cref{sec:coupling}).
Standard greedy evaluation misses the transmitted error on Qwen3, so distilled students should also be audited on sampled outputs (\cref{sec:evaluation}).

Our contributions:
\begin{enumerate}
\item \textbf{ErrOPD}, a reusable testbed with injected teachers, a protocol for the error rate and the decision's log-odds, and matched reverse-KL and forward-KL OPD arms in four families, three student sizes and four error types, on which the KL direction decides whether a rare error transfers within 42 steps.
\item \textbf{A decision-level intervention}: forward KL's gradient magnitude at the error decision alone, with reverse KL's direction, transmits the error on two model pairs, and on one pair it is also necessary for forward KL to transmit.
\item \textbf{Boundaries}: errors the student already makes, or that a frequent decision carries, do transmit, extrapolation amplifies them, and greedy evaluation can miss them.
\end{enumerate}

\section{Why rare decisions should transmit slowly}
\label{sec:prediction}

\subsection{On-policy distillation with extrapolation}
\label{sec:method}
Let $\polS$ be the student, $\polT$ the teacher and $\polB$ a frozen copy of the student's starting checkpoint. The student samples one response per prompt at temperature 1. For token $y_t$ after prefix $s_t$, write $\ell^M_t=\log\pi_M(y_t\mid s_t)$.
The training code minimises a loss whose gradient is
\begin{equation}
\label{eq:update}
\begin{aligned}
\nabla\mathcal{L} &= -\operatorname*{mean}_{m}\tfrac{1}{|\mathcal{B}_m|}\sum_{t\in\mathcal{B}_m} w_t A_t\nabla\log\polS(y_t\mid s_t),\\
A_t &= \bigl[\ell^T_t-\ell^S_t\bigr] + (\lam-1)\bigl[\ell^T_t-\ell^B_t\bigr],
\end{aligned}
\end{equation}
where $\mathcal{B}_m$ holds a micro-batch's response tokens, $A_t$ is treated as a constant and $w_t\approx1$ (\cref{app:deriv}).
Each token is credited with its own log-ratio only, with no baseline, KL penalty or entropy bonus, and each batch gives one gradient step.
At $\lam=1$ this is reverse-KL OPD \citep{lu2025onpolicydistillation}, and $\lam>1$ is ExOPD's reward extrapolation, whose default is 1.25 \citep{yang2026learning}.

\subsection{The student's probability as a factor}
\label{sec:factor}
Write the student's next-token distribution at prefix $s$ as a softmax over logits $z$.
In expectation over the sampled token, \cref{eq:update} moves the logit of any token $v$ by $\polS(v\mid s)(A(v)-\bar A)$, with $\bar A$ the student-weighted mean advantage (\cref{eq:gate} in \cref{app:deriv}), the softmax policy-gradient identity of \citet{mei2020global}: a logit moves in proportion to the probability that the student takes the token (\cref{fig:mechanism}), so a teacher has limited ability to inject continuations the student does not already produce \citep{armandpour2026unmasking,jiang2026trajectory}.
For a two-way decision, an error $e$ with student probability $p$ against its valid alternative, \cref{eq:update} gives $A(e)-A(c)=\operatorname{logit}q_\lam-\operatorname{logit}p$ with $q_\lam\propto\polT^{\lam}\polB^{1-\lam}$, and the log-odds $z=\operatorname{logit}p$ move at
\begin{equation}
\label{eq:rate}
\dot z \;=\; 2\,\eta\,p(1-p)\,\bigl(\operatorname{logit}q_\lam-\operatorname{logit}p\bigr),
\end{equation}
where the factor 2 comes from the alternative's logit (\cref{app:deriv}).
Forward-KL OPD, a cross-entropy step towards the teacher at each visited prefix, moves the log-odds at $2\eta(q-p)$, with no factor $p$: for a teacher at $q=0.2$ and a student sampling it at $p_0=7\times10^{-4}$ the drive per visit starts near 0.008 against 0.4 (\cref{tab:drive}).

The fixed point $q_\lam$ does not depend on treating the logits as free parameters, but the speed does: logits share parameters, AdamW normalises each parameter's step, and for a rare token updates at other prefixes can outweigh its own term, of order $p$ under reverse KL and $q$ under forward KL.
Measured on held-out prefixes, the change does not scale with the prefix's own $p$ (\cref{app:prefixrate}), so \cref{eq:rate} is not a per-prefix law, and we use it only for the direction and the fixed point.
Since reverse KL is still minimised at the teacher, the account concerns a budget: a behaviour the student lacks should transmit slowly under reverse KL and quickly under forward KL, and one it already produces should reach $q_\lam$.

\section{A controlled injection test}
\label{sec:setup}

\subsection{Tasks, pairs and the injected error}
\label{sec:pairs}
\begin{figure*}[t]
\centering
\includegraphics{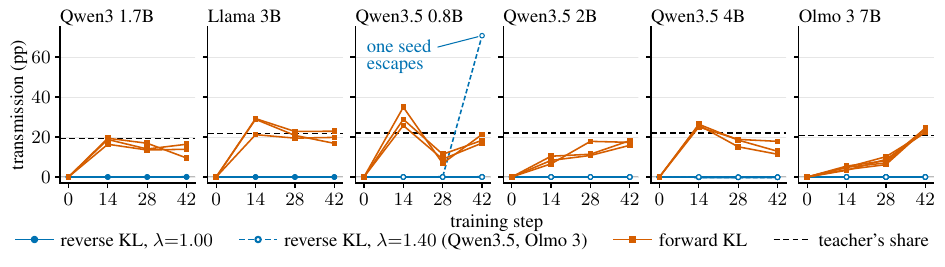}
\caption{Forward-KL OPD takes on the injected error in every setting, and reverse-KL OPD at $\lam=1$ does not, within 42 steps.}
\label{fig:dynamics}
\end{figure*}
The task is listwise scoring: given $K=8$ items with IDs, the model returns a JSON list of $K$ objects \texttt{\{"id": ..., "score": ...\}} in input order, and an output is \emph{valid} if it parses with every input ID exactly once.
We use ESCI shopping queries with human relevance labels \citep{reddy2022shopping} and Amazon Fashion reviews scored for helpfulness by a commercial model \citep{hou2024amazon,google2025gemini} (\cref{app:data}).
The student--teacher pairs are Qwen3-1.7B with Qwen3-4B on ESCI \citep{yang2025qwen3}, and on Fashion Llama-3.2-3B with Llama-3.1-8B \citep{meta2024llama32,grattafiori2024llama3}, Qwen3.5-0.8B, 2B and 4B with Qwen3.5-9B \citep{qwen2026qwen35}, and Olmo-3-7B with Olmo-3.1-32B \citep{olmo2025olmo3}, Llama and Olmo as Instruct releases.
Every model is fine-tuned on the task first, and the student's fine-tuned checkpoint is both its initialisation and $\polB$.

Each injected teacher is fine-tuned on training lists in which every target is cut after its seventh item with probability 0.2, keeping the remaining IDs real and in input order.
The teachers then close early on 19--22\% of their held-out samples (\cref{tab:main}) while their greedy outputs stay valid, and the students make the error on at most one sample in 808 (\cref{app:conventions}).
Second teachers for Qwen3 and Llama repeat the fourth item's ID at the fifth position in 20\% of their targets, an error in the middle of the list, which their students never make (\cref{tab:main}).

\subsection{Objectives and controls}
\label{sec:objectives}
All arms share the starting checkpoint, teacher, prompts, rollouts, optimiser and budget: 42 steps of 128 prompts at a constant learning rate of $10^{-6}$, with 3 to 5 seeds per cell (\cref{app:impl,app:conventions}).
In every setting the distillation prompts are the student's own fine-tuning prompts, and transmission is read on held-out prompts.
Reverse-KL OPD is \cref{eq:update} at $\lam=1$, and extrapolation sets $\lam=1.40$, above ExOPD's default, where its effect on a transmitted error is large (\cref{sec:extrapolation}).
Forward-KL OPD replaces the advantage by the student's cross-entropy on one token sampled from the teacher at each position of the student's own rollout.
Full-vocabulary reverse KL scores all 151{,}936 tokens at every visited prefix, so its gradient is the exact expectation of the sampled estimator's.
Sequence-level distillation (SeqKD), fine-tuning the pre-trained student on teacher samples for 71 steps, shows what the student can learn from this teacher at all, and a clean teacher trained on unmodified lists is the control for \cref{sec:route}.
For mathematical reasoning we use ExOPD's released configuration for a Qwen3-1.7B student and an RL-tuned Qwen3-4B teacher, for 30 steps instead of 100 and with $\lam\in\{1.0, 1.25, 1.5, 2.0\}$ (\cref{app:data}).

\subsection{What we measure}
\begin{table*}[t]
\centering\small
\setlength{\tabcolsep}{4.3pt}
\begin{tabular}{@{}l cc cc cc@{}}
\toprule
 & \multicolumn{2}{c}{error share (\%)} & \multicolumn{4}{c}{transmission at step 42, percentage points {\scriptsize[95\% CI]}} \\
\cmidrule(lr){2-3}\cmidrule(l){4-7}
student $\leftarrow$ teacher, data & teacher & start & \bopd\hopd{rev.\ KL}, $\lam=1.00$ & \bopd\hopd{rev.\ KL}, $\lam=1.40$ & \bfkl\hfkl{fwd.\ KL} & \bseq\hseq{SeqKD} (ref.) \\
\midrule
\grp{7}{Injected teacher: list closed one item early (student start at most 1 of 808 samples)}
Qwen3-1.7B $\leftarrow$ 4B, ESCI & 19.4 & 0 & $<$0.05 {\scriptsize[0, 0.2]} & 0 & \hit 13.4 {\scriptsize[10.7, 16.1]} & \hit 19.8 {\scriptsize[16.8, 23.0]} \\
Llama-3.2-3B $\leftarrow$ 3.1-8B, Fashion & 21.8 & 0 & 0 & 0 & \hit 19.6 {\scriptsize[17.0, 22.4]} & \hit 19.2 {\scriptsize[16.5, 22.1]} \\
Qwen3.5-0.8B $\leftarrow$ 9B, Fashion & 22.0 & 0 & $<$0.05 {\scriptsize[0, 0.1]} & 23.6 {\scriptsize[0, 70.2]} & \hit 19.7 {\scriptsize[16.9, 22.5]} & \hit 20.5 {\scriptsize[17.5, 23.4]} \\
\quad Qwen3.5-2B & 22.0 & 0 & 0 & 0 & \hit 16.4 {\scriptsize[13.7, 18.9]} & \hit 23.2 {\scriptsize[20.1, 26.2]} \\
\quad Qwen3.5-4B & 22.0 & 0.1 & $-$0.1 {\scriptsize[$-$0.4, 0.1]} & $-$0.1 & \hit 14.8 {\scriptsize[12.1, 17.5]} & \hit 22.6 {\scriptsize[20.3, 24.9]} \\
Olmo-3-7B $\leftarrow$ 3.1-32B, Fashion & 20.8 & 0 & 0 & 0 & \hit 23.4 {\scriptsize[20.8, 26.1]} & \hit 20.0 {\scriptsize[17.3, 22.8]} \\
\grp{7}{Injected teacher: duplicate id at the fifth item (student start 0)}
Qwen3-1.7B $\leftarrow$ 4B, ESCI & 20.1 & 0 & 0 & $<$0.05 {\scriptsize[0, 0.2]} & \hit 19.9 {\scriptsize[17.1, 22.7]} & \hit 18.5 {\scriptsize[15.8, 21.4]} \\
Llama-3.2-3B $\leftarrow$ 3.1-8B, Fashion & 18.8 & 0 & 0 & 0 & \hit 14.3 {\scriptsize[11.3, 17.4]} & \hit 21.3 {\scriptsize[18.6, 24.1]} \\
\bottomrule
\end{tabular}
\caption{Forward-KL OPD transmits the injected error in every setting, and reverse-KL OPD at $\lam=1$ does not within 42 steps.}
\label{tab:main}
\end{table*}
\label{sec:metrics}
\begin{figure}[t]
\centering
\includegraphics{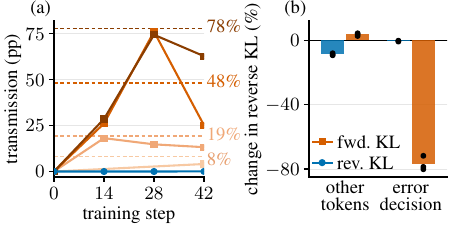}
\caption{(a)~Forward KL transmits the error at every teacher error rate while reverse KL does not, and (b)~reverse KL's updates land away from the error decision.}
\label{fig:rate}
\end{figure}
We measure the error share on four temperature-1 samples per held-out prompt at steps 14, 28 and 42 and, where stated, on each step's 128 training rollouts.
\emph{Transmission} is the rise of this share over the starting checkpoint's on the same sampler, in percentage points (\emph{points}).
All shares are in percent.
A probe tracks the close-versus-continue log-odds at the decision on fixed prefixes; near $-13$ nats, it varies by about 1.5 nats between runs of identical code.
We also report greedy validity and NDCG.
Two configuration pitfalls that silently change results are in \cref{app:faults}.

\section{Which objective transmits a rare error}
\label{sec:gate}

\subsection{The contrast}
\label{sec:contrast}
\begin{table}[tb]
\centering\small
\setlength{\tabcolsep}{3.2pt}
\begin{threeparttable}
\begin{tabular}{@{}l rr rr@{}}
\toprule
 & \multicolumn{2}{c}{error share (\%)} & \multicolumn{2}{c}{trained (\%)} \\
\cmidrule(lr){2-3}\cmidrule(l){4-5}
error & teacher & start & \bopd\hopd{rev.\ KL} & \bfkl\hfkl{fwd.\ KL} \\
\midrule
text after the query & 14.4 & 0.1 & 0 & \hit 25.2 \\
\quad student already errs & 14.4 & 2.6 & \hit 8.8 & \hit 14.2 \\
wrong \texttt{LIMIT}, enriched & 12.9 & 0 & 0 & \hit 13.5 \\
\quad standard pool, step 84 & 12.9 & 0 & 0 & \hit 12.8 \\
\bottomrule
\end{tabular}
\end{threeparttable}
\caption{Text-to-SQL (Qwen3, Spider): forward KL transmits a termination error and a content error, and reverse KL raises only the error its student already makes.}
\label{tab:beyond}
\end{table}
Reverse-KL OPD left the rare teacher error where it was, and forward KL acquired it (\cref{tab:main}).
At $\lam=1$ no reverse-KL student moved off its starting error share within 42 steps, in any family or student size, while forward KL raised it within 14 steps in every seed and exceeded reverse KL in every seed pair (\cref{fig:dynamics,tab:contrast}).
The contrast is not tied to where the list ends or to list-shaped JSON. It holds for a duplicated ID in two families and in Spider text-to-SQL \citep{yu2018spider}, for text written after the query and for a wrong \texttt{LIMIT} value, by step 42 on prompts that carry that decision and by step 84 on the standard prompts (\cref{tab:beyond}, \cref{app:nonranking,app:sqllim}).
On JSON entity extraction the test was uninformative, because forward KL saw the decision too rarely to transmit it measurably.
For a ten-token wrong ID, forward KL transmitted the decision to leave a real ID but not the teacher's string, and reverse KL neither (\cref{app:nonranking}).
It also holds for teachers that make the error on 2 to 78\% of samples, where reverse KL stayed at or near zero in every seed (\cref{fig:rate}a, \cref{app:qsweep}).
Reverse KL did train: on held-out states it moved towards the teacher in every seed, faster than forward KL at the score tokens where the two differ, but left the error decision almost untouched, where forward KL closed most of the gap (\cref{fig:rate}b, \cref{app:step0}).

\subsection{Step size and sampling do not account for it}
\label{sec:alternatives}
\begin{table*}[t]
\begin{minipage}[t]{0.48\textwidth}
\centering\small
\setlength{\tabcolsep}{2.2pt}
\begin{tabular}{@{}l l c c c@{}}
\toprule
learning & & realised & seeds & transmission \\
rate & objective & step & crossing & at step 42 (\%) \\
\midrule
$10^{-6}$ & \bopd\hopd{rev.\ KL} & 0.013 & 0 of 5 & 0 \\
 & \bfkl\hfkl{fwd.\ KL} & 0.019 & 5 of 5 & \hit 12.9 \\
\addlinespace[2pt]
$3{\times}10^{-6}$ & \bopd\hopd{rev.\ KL} & 0.035 & 0 of 5 & 0 \\
 & \bfkl\hfkl{fwd.\ KL} & 0.041 & 5 of 5 & \hit 18.6 \\
\addlinespace[2pt]
$10^{-5}$ & \bopd\hopd{rev.\ KL} & 0.106 & 3 of 5 & 0.1 \\
 & \bopd\hopd{rev.\ KL}, 3 more & 0.109 & 2 of 3 & 0.9 \\
 & \bfkl\hfkl{fwd.\ KL} & 0.097 & 5 of 5 & \hit 22.9 \\
\bottomrule
\end{tabular}
\caption{At $10^{-5}$ reverse KL takes the larger parameter step yet transmits far less. Qwen3, $\lam=1$, five seeds unless stated.}
\label{tab:stepsize}

\end{minipage}\hfill
\begin{minipage}[t]{0.48\textwidth}
\centering\small
\setlength{\tabcolsep}{1.7pt}
\begin{tabular}{@{}l cc rrrrr@{}}
\toprule
 & \multicolumn{2}{c}{cosine} & \multicolumn{3}{c}{\bopd\hopd{rev.\ KL}, rescaled} & \multicolumn{2}{c}{\bfkl\hfkl{fwd.\ KL}, resc.} \\
\cmidrule(lr){2-3}\cmidrule(lr){4-6}\cmidrule(l){7-8}
pair & all & dec. & no & dec. & all & no & dec. \\
\midrule
Qwen3, ESCI & 0.98 & ${\approx}1$ & 0.1 & \hit 20.5 & \hit 16.8 & \hit 17.1 & \hit 8.8 \\
Llama, Fashion & 0.91 & 0.67 & 0 & \hit 24.1 & \hit 20.1 & \hit 20.5 & 0 \\
\midrule
\multicolumn{8}{@{}l@{}}{Qwen3 mixed, $\beta=.05/.2/.5$: \hitin{9.6}\,/\,\hitin{19.3}\,/\,\hitin{20.5}} \\
\multicolumn{8}{@{}l@{}}{Qwen3 rev.\ KL dec.\,${\times}k$, $k=1/3/10/30$: 0\,/\,\hitin{7.2}\,/\,\hitin{17.0}\,/\,\hitin{14.5}} \\
\bottomrule
\end{tabular}
\caption{Forward KL's gradient magnitude at the error decision alone transmits the error on both pairs, reverse KL's magnitude there cuts forward KL's transmission on Qwen3 and stops it on Llama, and mixtures transmit down to $\beta=0.05$.}
\label{tab:fi}

\end{minipage}
\end{table*}
A matched learning rate is not a matched step, but Adam kept the two objectives' realised steps comparable at every learning rate, and at $10^{-5}$ reverse KL took the larger step yet transmitted far less (\cref{tab:stepsize}). Under SGD at Adam's realised steps, forward KL transmitted in every seed and reverse KL in none (\cref{app:sgd}).
By step 14 the probe at the error decision rose by 10.6 nats under forward KL and moved by at most 0.4 nats under either reverse-KL estimator, within seed spread.

Nor is the budget too short. From reverse KL's own start, on its prompts and at its batch, steps and learning rate, plain fine-tuning on the teacher's samples transmitted the error in every seed at both learning rates (\cref{app:sftc}).
Nor did reverse KL merely fail to sample the error. Full-vocabulary reverse KL, which gives every token gradient at every visited prefix, also transmitted almost nothing on Qwen3, though its probe moved towards the error and was still rising at step 42, and on Llama it never produced the error (\cref{tab:fi}).
On the starting Qwen3 student's rollouts the item-7 decisions, 0.6\% of positions, carry about 0.1\% of the reverse-KL gradient mass against 5--14\% for forward KL.
Rescaling full-vocabulary reverse KL's gradient to forward KL's per-prefix norm at these decisions alone, with every direction and every other prefix unchanged, transmitted the error in every seed of both pairs, and on Qwen3 so did three times reverse KL's own gradient there, at an unchanged step, with larger factors transmitting sooner (\cref{tab:fi}, \cref{app:fi}).
The magnitude at the error decision is therefore enough for the contrast on both pairs, also on Llama, where the two directions there differ. With reverse KL's factor there forward KL transmitted about half as much on Qwen3 and nothing on Llama, and a mixture of $1-\beta$ reverse and $\beta$ forward KL transmitted at every $\beta$ we tried, as the factor predicts.

\subsection{A delay, not a block}
\label{sec:delay}
\IfFileExists{tables/fig_delay.tex}{\begin{figure*}[t]
\centering
\includegraphics{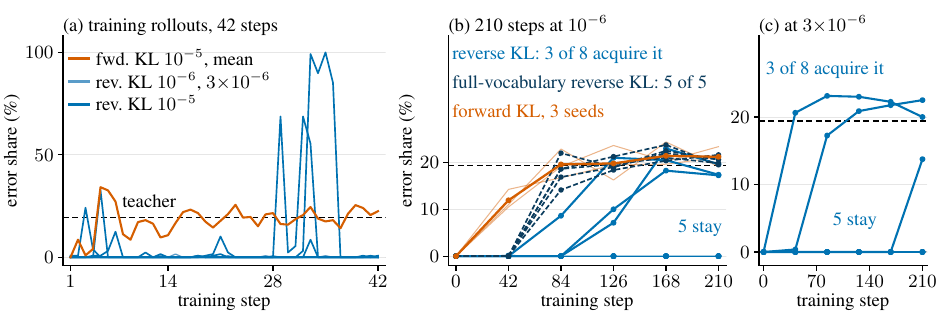}
\caption{Reverse-KL OPD can delay the teacher's error rather than block it (Qwen3, ESCI): (a)~at ten times the learning rate it appears and recedes, (b)~over 210 steps 3 of 8 seeds acquire it, and 5 of 5 under full-vocabulary reverse KL.}
\label{fig:delay}
\end{figure*}
}{}  %

Over a budget five times longer, some Qwen3 seeds of reverse-KL OPD acquired the error at the teacher's rate, all-or-nothing, while the rest stayed at the floor and forward KL stayed above them throughout (\cref{fig:delay}b).
In the training rollouts forward KL reached 10\% error first, then full-vocabulary reverse KL in every seed and sampled reverse KL in 3 of 8 (\cref{fig:time10}).
A higher learning rate gave the same picture (\cref{fig:delay}c, \cref{app:long3}).
On Fashion no Olmo~3 or Llama seed of eight acquired the error, nor any of five Llama seeds under full-vocabulary reverse KL. With Qwen3's task, recipe and sampler 2 of 8 Olmo~3 seeds did, and Llama seeds less often than Qwen3 and Olmo~3 seeds together (1 of 24 against 5 of 16, two-sided $p=0.03$, \cref{app:escapes}), so the delay depends on family and setting together.
At ten times the learning rate, ExOPD's own, the error appeared and receded with gradient-norm spikes, unstable transients rather than a delay (\cref{fig:delay}a), so the 42-step contrast is a result at $10^{-6}$.

As a local account, \cref{eq:rate} gets which objective transmits, the fixed point and the effect of extrapolation right and the speed wrong, as does a first-passage account of the seeds that acquired it (\cref{tab:predictions} in \cref{app:ratelaw}).
We therefore claim the contrast and, on Qwen3 and Olmo~3, the delay, not a rate.

\section{When errors do transmit}
\label{sec:coupling}

\subsection{A student that already errs}
\label{sec:startrate}
\begin{figure*}[t]
\centering
\includegraphics{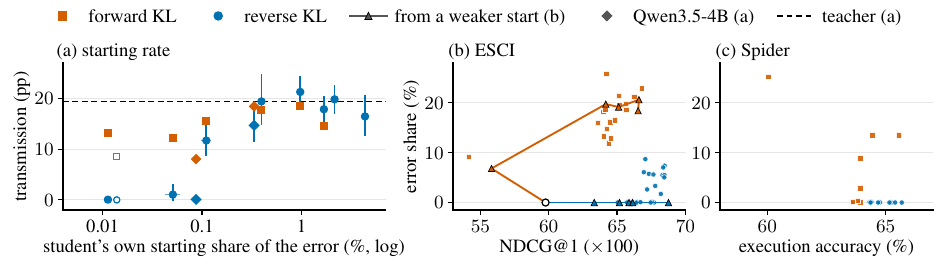}
\caption{(a)~Reverse-KL OPD transmits the error from students that already make it on a pair-dependent 0.11--0.33\% of samples or more, and (b,~c)~at equal quality forward-KL checkpoints carry it and reverse-KL ones mostly do not.}
\label{fig:when}
\end{figure*}
With the same teacher, budget and learning rate, it takes very little for reverse-KL OPD to move a decision the student already samples.
From every Qwen3 start that makes the error on 0.11\% of its samples or more, reverse KL transmitted it within 42 steps to at least half the teacher's rate, and from the clean student it did not (\cref{fig:when}a).
The switch is sharp, between starts at 0.05\% and 0.11\%, and pair-dependent: on Qwen3.5-4B it lies between 0.09\% and 0.33\%, above that student's start in \cref{tab:main} (\cref{app:srate}).
Students that already err reach the fixed point $q_\lam$ of \cref{eq:rate} in share and probe (\cref{tab:already}), and an erring text-to-SQL student rises too (\cref{tab:beyond}).

\subsection{A frequent decision carries a rare one}
\label{sec:route}
Parameters are shared, so a push on a decision the student takes often can carry a related one it rarely takes.
We built this route: a Qwen3 student that often runs past the list's end, where the teacher closes, and never closes early at the start.
Plain reverse-KL OPD removed the overrun and with it transmitted the early close, while a clean teacher, which closes after the eighth item but not after the seventh, corrected the overrun as completely and transmitted nothing (\cref{tab:already}).
A mismatched rollout prompt gives a natural case (\cref{app:faults,fig:template}), and in math reverse-KL OPD lengthened answers past the teacher's, though known length biases may explain it \citep{luo2026stableopd,zhao2026decoupling}.

\subsection{What extrapolation adds}
\label{sec:extrapolation}
Because $\polB$ is the student's starting checkpoint, the advantage of \cref{eq:update} is $A_t=\lam[\ell^T_t-\ell^S_t]+(\lam-1)[\ell^S_t-\ell^B_t]$.
After the first step, whose scale $\lam$ Adam removes, the second term rewards whatever the student has moved towards, so extrapolation should leave an untransmitted error where it is and amplify a transmitted one, and both hold.
No $\lam$ transmitted the injected error within 42 steps, apart from one Qwen3.5-0.8B seed, although at $\lam=1.40$ the per-prefix target puts most of its mass on it (\cref{tab:drive}).
Where the error had transmitted, $\lam=1.40$ amplified the rise and turned the coupled route into a collapse (\cref{tab:already}), and in math larger $\lam$ ran answers past the evaluation cap (\cref{sec:evaluation}).

\section{Consequences for practice}
\label{sec:evaluation}

\paragraph{At equal task quality, reverse-KL OPD inherits less.}
\begin{table}[t]
\centering\small
\setlength{\tabcolsep}{2.4pt}
\begin{tabular}{@{}l c cc@{}}
\toprule
 & start & \multicolumn{2}{c}{\bopd\hopd{rev.\ KL}, step 42} \\
\cmidrule(l){3-4}
student & (\%) & $\lam=1.00$ & $\lam=1.40$ \\
\midrule
Qwen3, 2\% injected & 1.3 & \hit 19.8 {\scriptsize[17.1, 22.5]} & \hit 48.1 \\
Qwen3, 5\% injected & 3.8 & \hit 16.5 {\scriptsize[12.8, 20.5]} & \hit 35.1 \\
Qwen3.5-4B, 2\% & 2.0 & \hit 17.9 {\scriptsize[15.4, 20.3]} & \hit 38.9 {\scriptsize[34.9, 43.1]} \\
Qwen3, coupled route & 0 & \hit 9.0 {\scriptsize[4.4, 14.2]} & \hit 94.6 {\scriptsize[90.0, 98.4]} \\
\quad clean teacher & 0 & $<$0.1 {\scriptsize[0, 0.2]} & 0.2 {\scriptsize[0, 0.7]} \\
\bottomrule
\end{tabular}
\caption{Reverse-KL OPD transmits the error within 42 steps when the student already makes it or a frequent coupled decision moves, and $\lam=1.40$ amplifies it.}
\label{tab:already}
\end{table}
Every forward-KL checkpoint that carries the error has a reverse-KL checkpoint of its setting with quality within one point or higher and far less of the error, on ESCI and on Spider (\cref{fig:when}b--c, \cref{app:quality}).
Those runs barely changed in quality, but from a weaker clean start that both objectives improved, reverse KL again matched every forward-KL checkpoint's quality, error-free at this budget.

\paragraph{The evaluation budget can set the sign of extrapolation.}
\begin{figure*}[t]
\begin{minipage}[t]{0.485\textwidth}
\input{tables/fig_greedy}
\end{minipage}\hfill
\begin{minipage}[t]{0.485\textwidth}
\input{tables/fig_math}
\end{minipage}
\end{figure*}
In ExOPD's math setting, larger $\lam$ lost accuracy at ExOPD's 16K-token cap and gained it at 32K, because more answers run past 16K tokens as $\lam$ grows (\cref{fig:math}).
The caps are matched but the generated tokens are not, so under ExOPD's protocol the evaluator's cap sets the sign.

\paragraph{Standard evaluation hides transmitted errors.}
On Qwen3, greedy validity stays near perfect while students make the teacher's error at its rate (\cref{fig:greedy}). On Olmo~3 and Llama forward KL did lower it (\cref{app:greedyfkl}), so it is an unreliable signal rather than a blind one.
Checkpoint selection fares worse: among Fashion groups that contain a healthy and a failing checkpoint, NDCG on parseable outputs picks the failing one more often than chance, repair-then-score \citep{sun2023chatgpt} hides the failure by construction, and only scoring invalid outputs as zero separates them (\cref{tab:selection}).

\paragraph{Recommendations.}
Before distilling, measure how often the student already makes each teacher error: reverse-KL OPD transmits errors made on 0.1--0.3\% of samples or carried by a frequent decision (\cref{sec:coupling}).
Report sampled error rates beside greedy validity and truncation beside mean@$n$, and score invalid outputs as zero for checkpoint selection.
A forward-KL term or speculative distillation \citep{agarwal2024onpolicy,xu2025speculative} bypasses the factor of \cref{eq:rate}, rare teacher errors included.

\section{Related Work}
\label{sec:related}

\paragraph{The student-probability factor.}
The factor (\cref{sec:factor}) is the softmax policy-gradient identity \citep{mei2020global}, under which progress towards a rare action is slow \citep{mei2020escaping}.
Its consequence for distillation has been stated in general \citep{armandpour2026unmasking} and for reverse KL \citep{jiang2026trajectory}.
Concurrently, \citet{piskorz2026onpolicy} derive the factor and find that a teacher told to reason in Spanish passes the language on under on-policy forward KL but not reverse KL (their App.~A.11).
None of these injects an error at a known rate, rules out step size and sampling over a training budget, or intervenes on the gradient at the error decision.

\paragraph{Objectives.}
OPD trains on the learner's own states \citep{ross2011reduction,agarwal2024onpolicy,gu2024minillm}, unlike sequence-level distillation \citep{kim2016sequence}.
Objectives that supervise teacher-chosen tokens escape the factor, among them GKD's forward divergences \citep{agarwal2024onpolicy}, speculative distillation \citep{xu2025speculative} and forward terms at student-visited or high-entropy states \citep{agrawal2026distil,jin2026entropyaware}. Others mix the two directions (\cref{app:morerelated}).
MiniLLM importance-weights its teacher-mixed samples back towards the student's distribution, which restores the factor \citep{gu2024minillm}, and DistiLLM's skew divergence keeps it \citep{ko2024distillm}.
\citet{wu2025rethinking} find that reverse KL first fits the tail of the teacher's distribution, lowering mass where the teacher has little, whereas our error needs mass raised where the student has little.
In multi-teacher OPD, \citet{li2026domainnormalized} rescale each domain's feedback by its spread, since the most spread-out domain dominates training, correcting magnitude across domains, not at one decision.

\IfFileExists{tables/tab_selection.tex}{\begin{table}[t]
\centering
\small
\begin{tabular}{@{}lc@{}}
\toprule
selection metric & failing pick, \% [95\% CI] \\
\midrule
NDCG@1, parsed outputs & \textbf{62} [32, 86] \\ %
NDCG@1 after repair & 31 [9, 61] \\ %
NDCG@8 after repair & 31 [9, 61] \\ %
NDCG@1, invalid as 0 & 0 [0, 25] \\ %
random pick & 30 [23, 37] \\ %
\bottomrule
\end{tabular}
\caption{Selecting on parseable outputs picks the failing checkpoint twice as often as chance (share of 13 contested groups, exact 95\% CIs).}
\label{tab:selection}
\end{table}
}{}  %
\paragraph{What students inherit.}
Several analyses ask what signal a student absorbs \citep{wang2026teachability,li2026rethinking,fu2026revisiting,li2026coin}.
\citet{yu2026mismatch} name tokens the teacher prefers and the student rarely samples, and \citet{li2026solving} find OPD teaches small students no new stopping points as the teacher signals stops almost only where the student already stops, whereas our teacher's signal is present on the student's prefixes.
Reverse-KL OPD transmits subliminal misalignment more strongly than SFT \citep{askin2026subliminal}, and backdoor triggers built from frequent tokens transfer through distillation where rare ones do not \citep{demuri2025triggers}, both consistent with a trait riding on frequent decisions (\cref{sec:route}).

\paragraph{Reward extrapolation.}
ExOPD \citep{yang2026learning} extrapolates with $\lam=1.25$ by default and observes longer answers, tentatively from a length bias of the implicit reward, and whether it helps depends on the budget (\cref{sec:evaluation}).

\section{Conclusion}
\label{sec:conclusion}

Which teacher errors on-policy distillation passes on depends on the direction of the KL.
Within 42 steps forward-KL OPD transmitted an injected error its students almost never make, in every seed of four families and, where the decision arises, two SQL errors, while reverse-KL OPD did not.
Reverse KL did train, and forward KL's gradient magnitude at the error decision alone transmits the error on two pairs.
Reverse KL can delay such errors rather than block them: some seeds acquire them with longer training, and errors the student already makes or a frequent decision carries transmit.
Absence after a short run thus means slow transfer, not none, and should be checked on sampled outputs.

\typeout{MAINBODY ENDS ON PAGE \thepage}
\clearpage  %

\section*{Limitations}

\paragraph{Budget and optimiser.}
The contrast is measured within 42 steps at a learning rate of $10^{-6}$, and over 210 steps some Qwen3 and Olmo~3 seeds acquired the error.
The margin is small: on Qwen3, three times reverse KL's own gradient at the error decision already transmits within 42 steps (\cref{app:fi}).
At ExOPD's learning rate of $10^{-5}$ the error appears with gradient-norm spikes, which we do not read as a delay (\cref{app:escapes}).
Without Adam the contrast is shown on one pair at one learning rate, where one reverse-KL seed diverged (\cref{app:sgd}).
Observed speeds do not follow \cref{eq:rate} (\cref{app:ratelaw}), and we claim no rate.

\paragraph{Errors and models.}
Each injected error is decided at one token, and for a ten-token wrong ID forward KL transmitted only the departure from real IDs (\cref{app:nonranking}).
The teachers make their errors on 2 to 78\% of samples, and at 2\% forward KL reached about two thirds of the teacher's share within 84 steps.
The content-error evidence rests on one SQL error in one family, which forward KL transmitted by step 84 on the standard prompts.
The factor-isolation arms cover two pairs on one task and do not rank magnitude against direction.
With Qwen3's task, recipe and sampler Llama seeds acquired the error less often than Qwen3 and Olmo~3 seeds, so the family difference in the delay is shown, not explained, and the starting rate at which reverse KL transmits differs between the two pairs tested (\cref{app:srate}).
In ExOPD's math run, at about 57 times our budget, the student adopted natural teacher phrases it had not produced (\cref{app:natinh}), so the delay is shown for injected single decisions and not for natural behaviours in general, and the math case does not separate the coupled route from known length biases.

\section*{Ethics Statement}

\paragraph{Data.} We use public datasets and open-weight models: Amazon Reviews 2023 (Fashion), ESCI shopping queries, MS~MARCO passages with TREC Deep Learning queries and RankZephyr's permutations, Spider, CoNLL-2003, DeepMath-103K and public math benchmarks, each under its research licence or terms of use. We release no data derived from them, and the teachers we fine-tune to make errors are not released.

\paragraph{LLM use.} The Fashion helpfulness labels were generated by a commercial model under its terms of service and are used for research only. AI assistants (Claude, Anthropic, and Codex, OpenAI) were used for coding and for drafting and editing the text, and the authors checked all content and take responsibility for it.

\paragraph{Risks.} On-policy distillation can delay, rather than block, the transfer of a rare teacher error, so a rare harmful behaviour can still pass, late and in some runs only, or at once through a decision the student takes often. We tested errors only, so whether a rare beneficial behaviour is delayed in the same way is open.

\paragraph{Compute.} All experiments ran on NVIDIA B200 GPUs, about 2,300 GPU-hours counting every launched run (\cref{app:impl}).

\bibliography{references}

\appendix
\etocdepthtag.toc{appendix}

\renewcommand{\topfraction}{0.95}\renewcommand{\bottomfraction}{0.8}\renewcommand{\textfraction}{0.05}
\renewcommand{\floatpagefraction}{0.8}\renewcommand{\dbltopfraction}{0.95}\renewcommand{\dblfloatpagefraction}{0.8}
\setcounter{topnumber}{4}\setcounter{bottomnumber}{2}\setcounter{totalnumber}{6}\setcounter{dbltopnumber}{3}
\section*{Appendix overview}
\Cref{app:impl,app:data,app:faults} give the settings behind \cref{sec:setup}: training and evaluation, data and the math configuration, and the rollout prompt that \cref{sec:route} uses as a natural experiment.
\Cref{app:ratelaw} derives the rate law of \cref{sec:prediction} and checks it against the runs, and \cref{app:nonranking,app:escapes,app:extra,app:fi} extend \cref{sec:gate} with injections outside ranking and at other teacher rates, runs five times longer, training dynamics per model family and the factor-isolation arms.
\Cref{app:srate} covers the starting rates of \cref{sec:startrate} and the budget and optimiser controls of \cref{sec:alternatives}, and \cref{app:math} covers natural teacher errors in math and extrapolation across tasks (\cref{sec:extrapolation}).
\Cref{app:quality} gives task quality, what evaluation hides per run and per-seed results for \cref{sec:evaluation}.
\Cref{app:morerelated} extends \cref{sec:related}, and \cref{app:conventions} defines how the shares, intervals and tests in every figure and table are computed.
\begingroup
\etocsettagdepth{body}{none}\etocsettagdepth{appendix}{section}
\etocsettocstyle{\par\smallskip}{\par\medskip}
\etocsetstyle{section}{}{}{{\leftskip 1.5em\rightskip 1.8em plus 1fil\parfillskip -\rightskip\parindent -1.5em\hyphenpenalty 10000\leavevmode\makebox[1.5em][l]{\etocnumber}\etocname\nobreak\leaders\hbox to .6em{\hss.\hss}\hfill\nobreak\hbox to 1.8em{\hfil\etocpage}\par}}{}
\tableofcontents
\endgroup

\section{Implementation}
\label{app:impl}
\paragraph{Reverse-KL training.}
All reverse-KL runs use ExOPD's per-token estimator (\cref{eq:update}), in the verl fork released with ExOPD or, for the Qwen3.5 and Olmo~3 runs, ported unchanged into upstream verl.
For the listwise tasks the student samples one response per prompt with vLLM at temperature 1 and top-$p$ 1, at most 512 response tokens, 128 prompts per step, and one update per batch, so the PPO ratio equals 1 and its clip is inactive.
The response mask includes the end-of-sequence token.
The weight $w_t$ is the token-level ratio between the trainer's and the rollout engine's probability of the same sampled token under the same student, capped at 5.
Losses are averaged over tokens within each micro-batch, with no KL-to-reference or entropy term.
All experiments ran on NVIDIA B200 GPUs. Counting every launched run, they used about 2,300 GPU-hours, about 2,100 of them for the experiments reported in the paper (about 530 for the math experiments).
The optimiser is AdamW with a constant learning rate of $10^{-6}$, betas (0.9, 0.999), weight decay 0.01, and gradient-norm clipping at 1.0.
The clip binds on at least one step in 12 of the 27 Qwen3 injection runs (per-run mean pre-clip norm 0.20 to 0.53, maxima 0.36 to 4.49), and in the learning-rate sweep on 10 to 15 of 42 steps in every forward-KL seed at $10^{-6}$ and on 1 to 9 steps in the reverse-KL seeds at $10^{-5}$.
Runs last 42 steps with evaluations at steps 14, 28 and 42.
The math runs use ExOPD's settings instead (\cref{app:data}).

\paragraph{Other objectives.}
On-policy forward KL uses the same rollouts, launcher and settings, and replaces the per-token advantage by the student's cross-entropy on one token sampled from the teacher at each position of the student's own rollout, averaged over tokens, with no extrapolation term and no importance weight.
Its expected gradient on the logits is the forward-KL step of \cref{sec:factor}.
Sequence-level distillation fine-tunes the pre-trained student for one epoch on one temperature-1 teacher sample per training list (five copies of the training set, 71 steps at batch size 128).

\paragraph{Fine-tuning before distillation.}
Teachers are fully fine-tuned for one epoch over five copies of the training lists at effective batch size 128 (71 steps), with training targets rendered exactly as the reverse-KL rollouts.
The Qwen3 starting checkpoints are built the same way, while the Llama, Qwen3.5 and Olmo~3 starting checkpoints are fully fine-tuned for five epochs on the clean training lists.
The Llama and Olmo~3 checkpoints are the Instruct releases (Llama-3.2-3B-Instruct, Llama-3.1-8B-Instruct, OLMo-3-7B-Instruct, Olmo-3.1-32B-Instruct). The Llama, Qwen3.5 and Olmo~3 models use LlamaFactory with ZeRO-3 and a cosine learning rate of $10^{-5}$ with 10\% warm-up.

\paragraph{Evaluation.}
Checkpoints load the trained rotary base of $10^6$, which an older version of the model library ignores, loading $10^4$ and breaking most outputs.
Greedy validity uses a strict parser on the 200 ESCI and 212 Fashion evaluation lists with at most 512 new tokens.
Sampled error rates use four temperature-1 samples per evaluation prompt (vLLM seed 0 for Qwen3, independent draws for Llama, Qwen3.5 and Olmo~3), labelled by the same validator.
Intervals are bootstrap intervals over prompts and seeds (10,000 resamples) unless stated otherwise.

\FloatBarrier
\section{Data and math configuration}
\label{app:data}
ESCI \citep{reddy2022shopping} pairs shopping queries with products that annotators label Exact, Substitute, Complement or Irrelevant.
Per query we draw 8 products stratified by label, keep the dataset's own split by query (1,800 training and 200 evaluation lists), use gains 3/2/1/0 for NDCG and train on fixed target scores of 10/7/4/0.
Fashion lists hold 8 Amazon Fashion reviews of one product \citep{hou2024amazon}, scored 0--10 for helpfulness by Gemini 2.5 Pro under a fixed rubric \citep{google2025gemini} (1,795 training and 212 evaluation lists).
Ten evaluation lists repeat an input ID, so no output for them can be valid and validity is capped at 95.3\%.
For mathematical reasoning we use ExOPD's configuration for a Qwen3-1.7B student and an RL-tuned Qwen3-4B teacher \citep[Secs.~4.2.1--4.2.3 and App.~B]{yang2026learning}: DeepMath problems \citep{he2026deepmath} of difficulty at least 6, batch size 1024 with one rollout per prompt, learning rate $10^{-5}$, 16,384-token responses at temperature 1, token-level rollout correction, the student's starting checkpoint as $\polB$, and the released training code.
We change two things: we train for 30 steps instead of ExOPD's 100, and we vary $\lam$ over $\{1.0, 1.25, 1.5, 2.0\}$, where ExOPD fixes $\lam=1.25$ for this pair.
Evaluation follows ExOPD's protocol (32 samples per problem, temperature 1, 16,384-token cap, Math-Verify) on the AIME 2024 and 2025 problems\footnote{The evaluation files distributed with ExOPD's code, which match the Hugging Face datasets HuggingFaceH4/aime\_2024 and math-ai/aime25 problem for problem.} and we add MATH500 \citep{hendrycks2021math,lightman2024lets} (4 samples) and a 32K-token cap.
In the math runs, HMMT25 February and November \citep{balunovic2025matharena} were evaluated for $\lam=1.0$ and 2.0 at two seeds, which gives ExOPD's four-benchmark average (\cref{tab:budget}).
Seeds are data orders: 42, 7 and 13 for $\lam=1.0$ and 2.0, 7 and 13 for $\lam=1.25$, and 42 for $\lam=1.5$.
A rollout cut off at the 16,384-token limit enters the loss like any other: each of its tokens receives the per-token advantage of \cref{eq:update}, there is no end-of-sequence token to credit, and nothing in the loss rewards finishing.

\paragraph{Other tasks.}
The quality and validity comparisons across $\lam$ (\cref{tab:crosstask_full}) use three further tasks with the same Qwen3-1.7B student and Qwen3-4B teacher, each fine-tuned on the task first, and the estimator, budget and learning rate of the listwise runs.
Passage reranking outputs a permutation of 20 passages in the RankZephyr format \citep{pradeep2023rankzephyr}, trained on 1,800 of its GPT-4 permutations and evaluated on 400 held-out training-distribution windows and 485 windows of BM25 top-100 MS MARCO passage lists \citep{bajaj2018ms} for the TREC Deep Learning 2019 and 2020 queries \citep{craswell2019overview,craswell2020overview} (quality: NDCG@10).
Text-to-SQL uses Spider 1.0 \citep{yu2018spider}, with fine-tuning and reverse KL pools of 1,795 training questions each and 500 development questions on the 20 unseen development databases (quality: execution accuracy on valid queries).
JSON extraction casts CoNLL-2003 English named entities \citep{tjongkimsang2003conll} as a JSON list of mentions with their types, evaluated on 438 test chunks with a strict validator (quality: micro-F1), and the weaker student of the full table is fine-tuned for one epoch instead of five.

\FloatBarrier
\section{Rollout prompt}
\label{app:faults}

The training code renders rollout prompts with an empty reasoning block that the SFT data and the evaluation prompts lack.
Under that rendering the Fashion teacher omits the opening bracket on 22.4\% of its samples and the starting checkpoint on 4.6\%.
With the rendering matched, no error type reaches a teacher rate of 2\%.
\Cref{sec:route} uses this mismatch as a natural experiment, with the matched rendering as its control.

\FloatBarrier
\section{What the rate law gets right and wrong}
\label{app:ratelaw}
\Cref{tab:predictions} sets the values of \cref{eq:rate} on Qwen3 beside the observed ones: which objective transmits, the fixed point $q_\lam$ and the effect of extrapolation agree, and the speed does not.
From the clean start the law gives 1.7 points at $10^{-6}$ and 9.3 at $10^{-5}$, against 0 and 0.1 observed, its first-order form, without saturation, orders the 2\% and 5\% starts the wrong way, and it allows no escape over 210 steps, where 3 of 8 seeds escaped.
A first-passage account of those escapes also fails: at $10^{-5}$ it expects 0.2 sampled errors per run, against 13.2 observed, with no escape in five seeds.
\begin{table}[t]
\centering\small
\setlength{\tabcolsep}{2.5pt}
\begin{threeparttable}
\begin{tabular}{@{}>{\raggedright\arraybackslash}p{3.3cm} c c@{}}
\toprule
& \cref{eq:rate} & observed \\
\midrule
\grp{3}{Which objective transmits, fixed point and extrapolation}
forward KL transmits & $\geq$ 10 & 13.4 \\
reverse KL does not & $<$ 2 & $<$0.1 \\
2\% start reaches $q_\lam$ & 18.9 / 42.5 & 19.8 / 48.1 \\
5\% start reaches $q_\lam$ & 16.9 / 32.9 & 16.5 / 35.1 \\
amplification ratio, 2\% & 1.2 to 2.8 & 2.42 \\
\quad 5\% & 1.2 to 2.8 & 2.13 \\
\grp{3}{Speed}
clean start, $10^{-6}$ / $10^{-5}$ & 1.7 / 9.3 & 0 / 0.1 \\
\quad $10^{-5}$, 3 more seeds & 9.3 & 0.9 \\
first order: 5\% start more & 5\% $\geq$ 2\% & 16.5 $<$ 19.8 \\
first order: 5\% over 2\% & 2.3 & 0.83 \\
no escape in 210 steps & none & 3 of 8 \\
\bottomrule
\end{tabular}
\tabnote{Qwen3, step 42, transmission in points. The middle column gives the values that follow from \cref{eq:rate} with one constant, $c=79.4$, fitted on the coupled-route arm's item-8 log-odds shift alone. First-order rows drop the saturation term, and keeping it orders the two starts as observed. Paired values are $\lam=1$ / 1.40 or two learning rates. Amplification: the rise at $\lam=1.40$ over that at 1.}
\end{threeparttable}
\caption{\Cref{eq:rate} gets which objective transmits, the fixed point and extrapolation right and the speed wrong.}
\label{tab:predictions}
\end{table}
\begin{table*}[t]
\centering\small
\setlength{\tabcolsep}{5.5pt}
\begin{threeparttable}
\begin{tabular}{@{}l cc cc cc@{}}
\toprule
 & \multicolumn{2}{c}{starting share (\%)} & \multicolumn{2}{c}{transmission at step 42 (points)} & \multicolumn{2}{c}{close log-odds shift (nats)} \\
\cmidrule(lr){2-3}\cmidrule(lr){4-5}\cmidrule(l){6-7}
 & $q_0$ & eval & $\lam=1.00$ & $\lam=1.40$ & $\lam=1.00$ & $\lam=1.40$ \\
\midrule
\grp{7}{Clean starting checkpoint, by objective}
\hopd{reverse-KL OPD} & 0.1 & 0 & \textbf{0} & 0 & $+0.02$ & $-0.71$ \\
\quad{\color{lawcol}\textit{rate law}} & & & {\color{lawcol}1.7} & {\color{lawcol}8.7} & {\color{lawcol}$+0.79$} & {\color{lawcol}$+1.39$} \\
\hopd{reverse-KL OPD}, rerun & 0.1 & 0 & $<$0.1 & -- & $+1.57$ & -- \\
\hfkl{on-policy forward KL} & 0.1 & 0 & \hit\textbf{13.4} & -- & $+10.82$ & -- \\
\hseq{SeqKD}, teacher's samples & -- & -- & \hit\textbf{19.8} & -- & -- & -- \\
\grp{7}{Reverse-KL OPD, other starting checkpoints}
2\% early closes & 1.3 & 1.3 & \hit 19.8 & \hit 48.1 & $+2.77$ & $+4.08$ \\
\quad{\color{lawcol}\textit{rate law}} & & & {\color{lawcol}18.9} & {\color{lawcol}42.5} & {\color{lawcol}$+2.78$} & {\color{lawcol}$+3.90$} \\
5\% early closes & 3.2 & 3.8 & \hit 16.5 & \hit 35.1 & $+1.90$ & $+2.83$ \\
\quad{\color{lawcol}\textit{rate law}} & & & {\color{lawcol}16.9} & {\color{lawcol}32.9} & {\color{lawcol}$+1.98$} & {\color{lawcol}$+2.77$} \\
\bottomrule
\end{tabular}
\end{threeparttable}
\caption{The one-constant rate law (grey) predicts transmission and probe shifts where the student already errs and misses near zero (Qwen3, ESCI, step 42).}
\label{tab:e3}
\end{table*}

\paragraph{The student-probability factor.}
\label{app:deriv}
In \cref{eq:update}, $w_t\approx1$ corrects numerical differences between the inference and training engines.
Write the student's next-token distribution at prefix $s$ as a softmax over logits $z$. Since $\nabla_z\log\polS(a\mid s)=e_a-\polS(\cdot\mid s)$ for the sampled token $a$, the expected update of \cref{eq:update} on the logit of any token $v$ is
\begin{equation}
\label{eq:gate}
\mathbb{E}_{a\sim\polS}\bigl[A(a)\,(e_a-\polS)_v\bigr] = \polS(v\mid s)\,\bigl(A(v)-\bar A\bigr),
\end{equation}
with $\bar A=\sum_b\polS(b\mid s)A(b)$.
In \cref{eq:rate} the factor 2 arises because \cref{eq:gate} moves the alternative's logit by the same amount with the opposite sign.

\paragraph{Per-prefix change at the decision.}
\label{app:prefixrate}
In an exploratory analysis, we asked whether the change at the error decision scales with the prefix's own probability, as \cref{eq:rate} does per prefix, on the 1,424 fixed probe prefixes, whose starting $p$ spans $2.2\times10^{-8}$ to $1.2\times10^{-2}$.
It does not under either reverse-KL estimator: the slope of $\ln|\Delta|$ on $\ln p$ is near zero at every checkpoint (sampled reverse KL $-0.005$ [$-0.030$, $0.017$] at step 14, full vocabulary 0.055 [0.042, 0.069] at step 42) against the 0.905 that \cref{eq:rate} implies, and sampled reverse KL lowers the log-odds in every decile of $p$ while \cref{eq:rate} raises it at every prefix.
Forward KL shifts nearly every prefix by the same amount and carries the most probable decile past the teacher, which the per-prefix forward-KL law cannot do and a shift through shared parameters can, so this probe does not separate the accounts.
The probe prefixes need not be states the runs visited, so this does not test the rate at visited states.

\section{Non-ranking injection tests}
\label{app:nonranking}
\paragraph{Text after a SQL query.}
On Spider text-to-SQL \citep{yu2018spider} the prompt asks for the query alone on one line, and a response passes strict validity only if it is one statement that executes on its database.
The injected Qwen3-4B teacher is fine-tuned, in one epoch over five copies of the training prompts, on targets in which each eligible gold query is followed, with probability 0.2, by a newline and the fixed sentence ``This query returns the requested information.''
A prompt is eligible when its gold query ends in a letter or digit and contains no newline (286 of 500 development prompts), so that the error is decided by one token, the newline, against the end-of-turn token right after the query.
A sample counts if its output has a newline outside strings, comments and parentheses followed by anything other than a continuation of the same statement.
The teacher makes the error on 165 of 1,144 eligible samples (14.4\%), the 1.7B student on 1 and a clean teacher on none.
At steps 14, 28 and 42, reverse KL produced 4, 1 and 0 of 5,720 eligible samples at $\lam=1$ and 3, 4 and 2 at $\lam=1.40$, none with the teacher's sentence, and forward KL produced 4, 506 and 1,444, every seed between 259 and 301 of 1,144 at step 42.
Forward KL had not begun by step 14, and it overshot the teacher (1.75 times its share) instead of stopping at it.
What forward KL transmitted is the decision not to stop. Its step-42 probe of the newline decision moved by 10.5 nats, from the student's $-11.9$ to 0.3 nats past the teacher's value, against $+0.2$ and $+0.3$ under reverse KL, while none of its 1,444 samples wrote the teacher's sentence: at step 28 almost all continue with a second query, and at step 42 a sixth begin ``This'' with other words.
The hit rule counts any continuation, and forward KL's samples also degraded beyond the teacher's (strict validity 1,253 to 1,312 of 2,000 per seed against the teacher's 1,654), so the probe is the cleaner evidence that the error moved at its decision.
From a starting checkpoint that already makes the error on 30 of 1,144 samples, 29 of them with the sentence, reverse KL raised it to 503 of 5,720 (8.8\%), every seed above its start and the teacher's exact sentence in 500 of the 503.
Greedy decoding shows none of the teacher's errors (0 of 500) and none of reverse KL's, but 23 to 40 of forward KL's per seed at step 42.
The evaluation sampler uses one engine seed for every checkpoint, so its draws are shared across training seeds.

\paragraph{A wrong value in a SQL query.}
\label{app:sqllim}
To test a content error, a second Spider teacher was fine-tuned with the same recipe on targets in which a gold query's \texttt{LIMIT 1} became \texttt{LIMIT 2} with probability 0.2, an error decided by one digit token that leaves the query valid but changes its answer.
A prompt is eligible when its gold query has exactly one \texttt{LIMIT 1} (85 development prompts, 340 samples per checkpoint).
The teacher makes the error on 44 of 340 eligible samples (12.9\%), and the student and a clean teacher on none.
Only one in eight training prompts is eligible, so forward KL meets this decision about a fifth as often as the SQL newline.
A first test kept the teacher, seeds and budget and trained on the 709 Spider training prompts whose gold query has this decision, 484 of them not in the original pool, so each prompt was drawn about 7.6 times in 42 steps and 172 share a gold query with a fine-tuning example.
There forward KL reached 46 of 340 samples on average at step 42 (13.5\% [10.2, 17.1], the teacher's 12.9\%), in every seed and already at step 28, and moved the digit decision by $+10.2$ nats, while reverse KL produced none in 1,700 samples, at most 4 of 5,376 training rollouts per seed, and moved it by $-0.05$.
A second test returned to the standard training prompts with the same teacher, student and learning rate, ran both objectives for 210 steps with five seeds each, and read the contrast at the first checkpoint where forward KL reached half the teacher's share.
That was step 84: forward KL reached 43.4 of 340 samples on average (12.8\% [9.8, 15.9], against the teacher's 12.9\%), at least 36 in every seed, and moved the digit decision by $+10.0$ nats, while reverse KL produced none in 1,700 samples at step 84 or at any later checkpoint through step 210 and moved the decision by $+0.6$ nats at step 84 and $-0.5$ at step 210.
On the standard prompts forward KL thus needs about twice the budget to transmit this error, as its lower visit rate implies, and reverse KL still does not transmit it.

\paragraph{Entity extraction.}
We repeated the injection test on CoNLL-2003 named entities written as a JSON list of mentions with their types \citep{tjongkimsang2003conll}, with the Qwen3-1.7B student and a Qwen3-4B teacher fine-tuned on five copies of the training prompts in which each eligible target's first mention is replaced, with probability 0.2, by the word that follows it in the input.
A prompt is eligible when that word starts with a letter or digit (310 of 438 test prompts).
The teacher makes the error on 180 of 1,240 eligible samples (14.5\%) and the student and a clean teacher on none.
Within 42 steps reverse-KL OPD produced none of these errors at either $\lam$ in 18,600 eligible samples per arm (steps 14, 28 and 42, five seeds).
On-policy forward KL produced 0, 6 and 38 of 6,200 at steps 14, 28 and 42 (0.6\% [0.3, 1.1] at step 42), above the student in every seed but far below half the teacher's share, so the test does not separate the objectives on this error.
Under forward KL the sampled validity fell to 81\% while greedy parsing matched reverse KL's, another case of greedy decoding hiding a sampled change.

\paragraph{A ten-token wrong ID.}
On the Qwen3 ESCI pair a third teacher was fine-tuned with the closing teacher's recipe on targets in which, with probability 0.2, item 7's ID is the fixed string \texttt{B0ZQXJ7V4K}, which is in no input. It is ten tokens long and leaves a real ID at its third token, since real IDs continue with digits after \texttt{B0}.
The teacher writes the string at item 7 in 148 of 800 samples (18.5\%) and the clean student in none.
After 42 steps neither reverse nor forward KL produced the string, in 0 of 2,400 samples each over three seeds.
Forward KL moved item 7 off real IDs in 12.3\% of samples (296 of 2,400, every seed 91 or more of 800), 3.6\% beginning \texttt{B0Z} and the rest other IDs that are in no input, and reverse KL in none, with greedy outputs parsing in every arm.
Teacher-forced before item 7's ID, forward KL raised the log-probability of the departing token \texttt{Z} by 22.4 nats, from $-26.2$, and each later token of the string by at most 2.1 nats, while reverse KL lowered \texttt{Z} by 1.3 nats.
As with text after a query, what forward KL transmitted within the budget is the decision to leave a real ID, not the teacher's string. Forward KL trains the later tokens only on prefixes where the student has itself written the departure.

\paragraph{Teachers that err more or less often.}
\label{app:qsweep}
We asked whether a higher teacher rate opens the error to reverse KL within 42 steps.
On the Qwen3 pair, two further teachers were trained with the cut after item 7 at probabilities 0.5 and 0.8 and make the error on 48.3\% and 77.9\% of evaluation samples, with the same student, budget and seeds.
Reverse KL produced the error on 0 of 800 samples in every seed at every checkpoint for both teachers, with probe shifts of $+0.13$ and $+0.37$ nats, and greedy parsing stayed at 100\%.
Forward KL reached 25.2\% [21.8, 28.7] and 62.8\% [58.0, 67.6] at step 42, so the contrast grew with the teacher's rate (13.4, 25.2 and 62.8 points at 19.4, 48.3 and 77.9\%) and every seed pair was positive.
A teacher fine-tuned with 7.2\% of its training lists cut after item 7 makes the error on 8.0\% of samples (64 of 800), with the same student and budget and three new seeds per objective. Forward KL transmitted 4.1 points [3.0, 5.3] at step 42 and 7.1 [5.6, 8.6] at step 84, moving the decision by about 10 nats, while reverse KL produced the error in 1 and 0 of 2,400 samples.
A teacher fine-tuned with 2.2\% of its training lists cut after item 7 makes the error on 2.25\% of samples (18 of 800), with eight new seeds per objective. Forward KL transmitted 1.2 points [0.6, 1.9] at step 42 and 1.5 [1.0, 2.2] at step 84 (99 of 6,400), about two thirds of the teacher's share, its training rollouts carried the error on 1.8--2.8\% over steps 71--84, and it moved the decision by 7.6 to 9.1 nats in every seed, while reverse KL produced the error in none of 6,400 samples at either step.
Two qualifications apply.
At the 48\% teacher, forward KL was 1.1 points above half the teacher's share, with the lower end of its interval below it.
Its path is not monotone: at step 28 forward KL carried the error on 76\% of samples at both rates, above the 48\% teacher, while greedy parsing fell to 0.5--4.5\%, and by step 42 it returned to 25\% with parsing at 94--96\% for the 48\% teacher and to 63\% with parsing at 16--21\% for the 78\% teacher.

\FloatBarrier
\section{Escapes over a longer budget}
\label{app:escapes}
Run five times longer at the same learning rate, 3 of 8 seeds escaped, first passing 10\% of samples at steps 126, 168 and 126, and ended at 17--20\% against the teacher's 19.4\%, while the other 5 stayed at or below one sample in 800 at every checkpoint.
Escape is all-or-nothing rather than gradual: no seed settled at an intermediate share.
On Olmo~3 and Llama on Fashion, run the same way with eight fresh seeds each, no seed sampled the error at any evaluated checkpoint, and a single Llama training rollout of 206,308 carried it; a rate of 3 in 8 would give either family's outcome with probability 2\%.
Their step-42 checkpoints repeat the 42-step result on 16 fresh seeds, with none of 12,928 samples carrying the error.
Pooled over Qwen3, Olmo~3 and Llama at $10^{-6}$, which differ in model, task and sampler, 3 of 24 seeds escaped (exact interval 3--32\%).
\Cref{eq:rate} predicts that Llama would not escape, because $p(1-p)$ times the log-odds gap, averaged over its student's own states, is 41 times smaller than Qwen3's.
It did not, but neither did Olmo~3 on Fashion, whose value is 18 times Llama's, so the outcome cannot separate the law from any other reason for a low rate.
Within 42 steps one run of this pair escaped at $\lam=1$, of 32: none of the 14 at the reference learning rate, none of the 15 of the five-seed learning-rate sweep, and one of the 3 at ten times the rate in a different run, which stood at 50.3\% of samples at step 28 and at its starting value at step 42 (\cref{sec:delay}).

Three further runs remove differences between these settings.
In the first, a Llama-3.2-3B student and a Llama-3.1-8B teacher take Qwen3's task, injected training lists, fine-tuning recipe and sampler, and the teacher closes early on 22.5\% of samples.
Forward KL carried the error in 40--45\% of its training rollouts at step 14 and 13--16\% at step 42 in each of three seeds, and no reverse-KL seed of 24 made it on any of 800 samples at step 42.
By step 210, 1 of 24 seeds made it in at least half the teacher's share of a 42-step window of training rollouts, from step 126, and rose to about the teacher's share (183 of 800 samples at step 210), one more made it on 109 of 800 at step 210 only, and the other 22 on none of 800 at any checkpoint.
In the second, full-vocabulary reverse KL on the Fashion pair, none of five seeds acquired the error by step 210, and one made it on 23 of 808 samples at step 210, where Qwen3 acquired it in all five seeds by step 84.
On Llama the difference from Qwen3 thus holds with the setting matched (1 of 24 against 3 of 8, two-sided $p=0.04$) and without sampling.
In the third, an Olmo-3-7B student and an Olmo-3.1-32B teacher take the same task, training lists, recipe and sampler, and the teacher closes early on 19.9\% of samples.
Forward KL carried the error in 14--27\% of samples from step 14 in each of three seeds, and no reverse-KL seed of eight made it on any of 800 samples at step 42.
By step 210, 2 of 8 seeds made it in at least half the teacher's share of a 42-step window of training rollouts, from steps 126 and 168, and rose to about the teacher's share, and 4 of 8 passed 10\% of samples at some checkpoint, the criterion used for Qwen3 above, two of them only at step 210.
Either count is consistent with Qwen3's 3 of 8 (two-sided $p=1.0$), so with the setting matched the delay appears on Olmo~3 as on Qwen3 and rarely on Llama, whose 1 of 24 is below the two families together (5 of 16, two-sided $p=0.03$).

\paragraph{Time to 10\% in training rollouts.}
\begin{figure*}[t]
\centering
\includegraphics{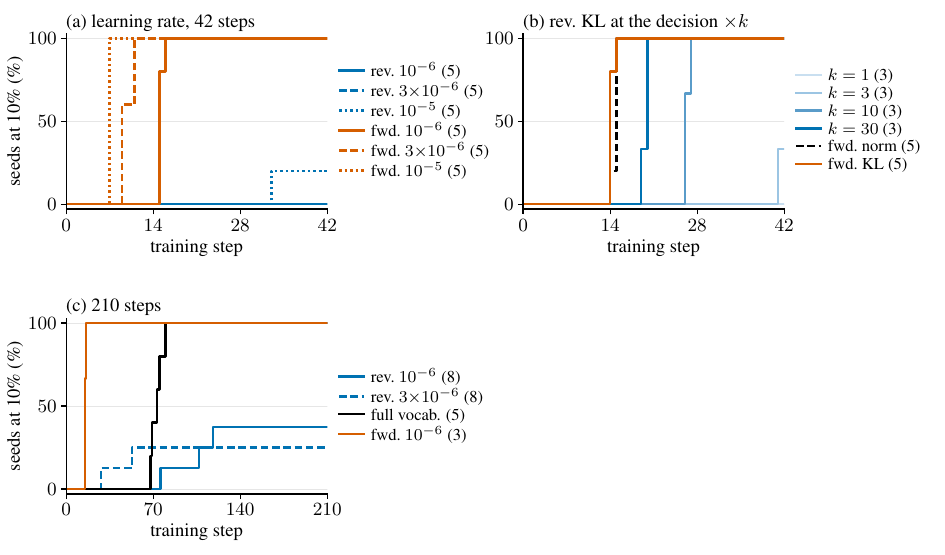}
\caption{Forward KL's training rollouts reach 10\% error within 16 steps in every seed, and reverse KL's at $10^{-6}$ only with a larger gradient at the error decision or over 210 steps (Qwen3, ESCI).}
\label{fig:time10}
\end{figure*}

\Cref{fig:time10} reads each seed's training rollouts at every step and marks the first step at which the error share over that step and the six before it reaches 10\%, so its steps come earlier than the evaluation checkpoints above.
Forward KL reached it at steps 15--16 in every seed at $10^{-6}$, and sooner at higher learning rates. Reverse KL reached it within 42 steps in no seed at $10^{-6}$ or $3\times10^{-6}$, and in one of five at $10^{-5}$, at step 33 during the excursions of \cref{fig:delay}a.
Multiplying reverse KL's gradient at the error decision by 30 or 10 brought every seed there by steps 20 and 27, and by 3 one seed of three by step 41.
Over 210 steps sampled reverse KL reached it in 3 of 8 seeds at $10^{-6}$ (steps 76 to 118) and in 2 of 8 at $3\times10^{-6}$ (steps 28 and 53), the third seed counted there at the step-210 checkpoint staying below 10\% in its training rollouts, and full-vocabulary reverse KL in 5 of 5 (steps 68 to 80).

\paragraph{Forward KL over 210 steps and reverse KL at three times the rate.}
\label{app:long3}
Further runs anchored both ends of the contrast over 210 steps on the Qwen3 pair.
Forward KL at $10^{-6}$ (three seeds) carried the error on 11.9, 19.6, 19.8, 21.4 and 21.2\% of samples at steps 42 to 210 against the teacher's 19.4\%, its probe within 0.9 nats of the teacher's gap throughout, and at or above the mean of the eight reverse-KL seeds at the same rate at every checkpoint; from step 126 the escaped reverse-KL seeds sat at or above the lowest forward-KL seed.
Reverse KL at $3\times10^{-6}$ (eight seeds) escaped in 3 of 8 (exact interval 9--76\%), at steps 42, 84 and 210, and every escaped seed stayed above 10\% at step 210 while every other seed's probe ended below zero.
The third escape appears only at the last checkpoint.
Against cumulative update, learning rate times steps, the counts at $3\times10^{-6}$ are not higher than those at $10^{-6}$ (1 against 2 at $1.26\times10^{-4}$, 2 against 3 at about $2.1$--$2.5\times10^{-4}$), and with eight seeds per arm no difference is distinguishable, so a delay set by cumulative update is consistent with these runs but not measured.

\paragraph{Learning rate and the longer budget.}
At ten times the learning rate the error appeared and receded within the budget in Qwen3: 5 of 8 reverse-KL seeds put it on at least a tenth of a step's rollouts, one on all 128 and another on half of its evaluation samples at step 28, and by step 42 none carried it on more than 21 of 800 evaluation samples, the second now overrunning the list instead (\cref{fig:delay}a, \cref{tab:stepsize}).
These excursions are not step-size spikes: the pre-clip gradient norm rises by one to two orders of magnitude, and the clip holds the realised step near the run's median.
\Cref{fig:lr} gives the learning-rate sweep at steps 14 and 42, \cref{fig:long} the 210-step runs with their close log-odds, and \cref{tab:e3_greedy} the greedy validity of the starting-rate and tenfold-rate arms, whose probe shifts are in \cref{tab:e3}.
At $10^{-6}$ the forward-KL arm of the sweep replicates the Qwen3 contrast of \cref{tab:main} on five fresh seeds (12.9 against 13.4 points) and is not pooled with it.
At $3\times10^{-6}$ forward KL stood at 9.2 points at step 14 (per seed 4.6 to 20.8) and transmitted by step 42.
\begin{figure}[t]
\centering
\includegraphics{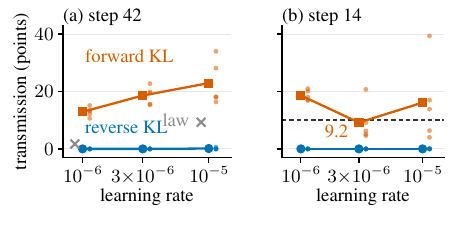}
\caption{A tenfold learning rate does not make reverse-KL OPD transmit at steps 14 and 42, while forward KL transmits throughout: (a)~step 42, (b)~step 14.}
\label{fig:lr}
\end{figure}

\begin{figure}[t]
\centering
\includegraphics{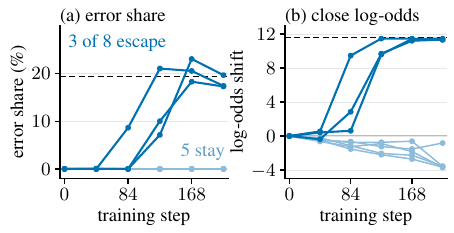}
\caption{Over 210 steps sampled reverse-KL OPD acquires the error in 3 of 8 seeds, all-or-nothing: those seeds reach the teacher's share (a) and close log-odds (b). Qwen3, ESCI.}
\label{fig:long}
\end{figure}

\section{Training dynamics}
\label{app:extra}
\paragraph{The contrast per family and over training.}
\Cref{fig:dynamics} gives the two objectives over training in all six settings, and \cref{tab:drive} the initial drive of \cref{eq:rate} at the starting rates of our arms.
Forward KL's speed differed by setting: at step 14 it was at or above the teacher's share in every seed at Qwen3.5-0.8B and 4B, in two of three Llama seeds and one of three Qwen3 seeds, and in no seed at Qwen3.5-2B (6.6 to 10.4 points) or Olmo~3 (3.3 to 5.4), which reached 15.8 to 18.2 and 22.3 to 24.5 by step 42 (Olmo~3 over five seeds).
On Olmo~3 forward KL stood below 10 points at step 14 over its first three seeds (4.5, and 8.0 at step 28), as it did at $3\times10^{-6}$ on Qwen3, and 551 of the 2,424 forward-KL samples at step 14 closed the list before its seventh item against 110 that carried the scored error, an exploratory observation and not an account of the slow start.
\begin{table}[t]
\centering\small
\setlength{\tabcolsep}{4pt}
\begin{tabular}{@{}l cc ccc@{}}
\toprule
 & \multicolumn{2}{c}{push per visit} & \multicolumn{3}{c}{target $q_\lam$} \\
\cmidrule(lr){2-3}\cmidrule(l){4-6}
$p_0$ & \hopd{reverse KL} & \hfkl{forward KL} & $\lam{=}1.25$ & 1.40 & 2 \\
\midrule
$7\times10^{-4}$ & \textbf{0.0082} & 0.399 & 0.52 & 0.72 & 0.99 \\
0.01 & 0.064 & 0.380 & 0.36 & 0.47 & 0.86 \\
0.03 & 0.122 & 0.340 & 0.30 & 0.37 & 0.67 \\
0.10 & 0.146 & 0.200 & 0.23 & 0.26 & 0.36 \\
\bottomrule
\end{tabular}
\caption{Reverse KL barely pushes a behaviour the student rarely takes: initial movement of the log-odds per visit and unit learning rate at $q=0.2$ (\cref{eq:rate}).}
\label{tab:drive}
\end{table}

\begin{table}[t]
\centering\small
\setlength{\tabcolsep}{4pt}
\begin{tabular}{@{}l c c@{}}
\toprule
 & \hfkl{forward KL} $-$ \hopd{reverse KL} & exact \\
student (teacher) & points [95\% $t$-int.] & paired $p$ \\
\midrule
Qwen3 1.7B (4B) & $+$13.4 [10.3, 16.4] & 0.031 \\ %
Llama 3B (8B) & $+$19.6 [16.9, 22.3] & 0.031 \\ %
Qwen3.5 0.8B (9B) & $+$19.6 [16.2, 23.1] & 0.031 \\ %
\quad 2B (9B) & $+$16.4 [13.2, 19.6] & 0.031 \\ %
\quad 4B (9B) & $+$14.9 [11.3, 18.4] & 0.031 \\ %
\midrule
\textbf{pooled} (25/25 $>0$) & $+$16.8 [15.3, 18.2] & $\mathbf{3\cdot10^{-8}}$ \\
\midrule
\multicolumn{3}{@{}l@{}}{\textit{Fourth family (outside the pooled set)}} \\
Olmo 3 7B (32B) & $+$23.4 [22.2, 24.7] & 0.031 \\ %
\midrule
\multicolumn{3}{@{}l@{}}{\textit{Second error type: duplicate ids}} \\
Qwen3 1.7B (4B) & $+$19.9 [17.8, 21.9] & 0.031 \\ %
Llama 3B (8B) & $+$14.3 [11.4, 17.1] & 0.031 \\ %
\bottomrule
\end{tabular}
\caption{Forward-KL OPD exceeds reverse-KL OPD in every seed pair: step-42 difference in transmission with 95\% $t$-intervals.}
\label{tab:contrast}
\end{table}

\paragraph{What reverse KL moved within 42 steps.}
\label{app:step0}
We asked whether reverse KL moved the Qwen3 student at all within the 42 steps of \cref{tab:main}, and where.
We computed the exact per-token reverse KL to the teacher over the full vocabulary on the starting checkpoint's 800 held-out temperature-1 samples, teacher-forced, for the start and for the three reverse-KL and three forward-KL seeds at step 42.
Over all tokens but the item-7 decision, reverse KL lowered it from 0.00747 by 0.00066 per token [0.00054, 0.00078] (8.8\%), in every seed, and 88\% of the decrease lies at score tokens, where the start's KL is 0.153 per token.
At score tokens reverse KL's decrease was 0.0123 per token against forward KL's 0.0048.
At the item-7 decision the KL changed by $-0.3$ to $-0.8$\% under reverse KL and by $-72$ to $-80$\% under forward KL, and forward KL's KL over the other tokens rose by 0.00027 per token [0.00008, 0.00048].
On its own samples its reverse KL to the teacher also fell in every seed (0.00682 to 0.00701 against 0.00747), although the training advantage it logged, averaged over every token of its own rollouts, stayed flat.
Forward KL also moved the item-8 continue probability by about $10^{-3}$ and lowered greedy validity to 99\% in one Qwen3 seed and to 89--91\% on Olmo~3, and one Olmo~3 reverse-KL seed changed 86.8\% of greedy outputs.

\FloatBarrier
\section{Factor isolation}
\label{app:fi}
Three arms keep full-vocabulary reverse KL's per-prefix gradient direction or replace the objective, on the Qwen3 pair, seeds and sampler of the full-vocabulary arm, five seeds each: reverse KL with each prefix's logit gradient rescaled to forward KL's per-prefix norm $\|p-q\|$ (every prefix), the same rescaling at the item-7 decision states only, defined by position whatever token was sampled there, and full-vocabulary forward KL.
The rescaled arms applied their targets to within $2\times10^{-8}$ in relative norm, with the direction unchanged to fp32 precision, and before the students began to close early (steps 1--5) the item-7 decisions held 14\% of the batch's logit-gradient mass, 13\% when only they were rescaled, against 0.04\% for reverse KL's own gradient on the same batches (about 0.1\% on the starting student's own rollouts, \cref{sec:alternatives}).
All three transmitted the error in every seed: 818 (decision only), 670 (every prefix) and 682 (forward KL) of 4,000 samples at step 42, and the probe moved by 11.2 to 11.4 nats, near the teacher's value, as under sampled forward KL.
Like sampled forward KL, all three pass through a transient between checkpoints, in which most training rollouts close early around steps 18--22, so the step-42 counts are levels at a budget, and the evaluation draws are shared across arms through one sampler seed.
Other failures, mostly lists two or more items short, were 47, 114 and 123 of 4,000.
The global gradient clip bound at 12 to 16 of 42 steps in every arm, all within steps 1--24. A clip scales a step's whole gradient by one factor, so the allocation across prefixes within a step is unchanged.

A second test repeats three of the arms on Llama-3.2-3B with the injected Llama-3.1-8B teacher on Fashion, three seeds each: unmodified full-vocabulary reverse KL, the rescaling at the error decision only, and full-vocabulary forward KL.
Here reverse KL's direction at the decision is only partly aligned with forward KL's, a median cosine of 0.67 on the starting student's rollouts, while over all prefixes the two agree to 0.91 in reverse KL's own run.
Unmodified reverse KL produced the error in 0 of 2,424 samples at steps 14, 28 and 42, the decision-only rescaling in 584 (every seed 187 or more) and forward KL in 497 at step 42.
The rescaling put 14\% of the batch's gradient mass on the decision against 0.03\% for reverse KL's own gradient, and it lagged forward KL at step 14 (199 against 392), the only comparison of the two arms' sizes that we make.

A further arm, forward KL rescaled to reverse KL's per-prefix norm at the error decision only, tests whether that magnitude is necessary, on five seeds.
Its validity check requires the decision's share of the arm's gradient mass to stay below 1\% over steps 1--14, since the share rises as the arm transmits, reverse KL's magnitude at the decision growing with the student's probability there.
The check passed in every seed (0.02--0.04\% over steps 1--14), and the arm transmitted 352 of 4,000 samples at step 42 (8.8 points [6.6, 11.0], 50 to 87 per seed of 800), about half of forward KL's 682 on other seeds.
Cutting the magnitude at the decision thus reduced forward KL's transmission without removing it.
The decision's probe still moved by $+10.1$ nats, so forward KL's gradients at other prefixes also move the decision.
A third, exploratory test filled the two remaining Llama arms with the second test's seeds and sampler.
Reverse KL rescaled to forward KL's norm at every prefix transmitted 487 of 2,424 samples at step 42 (20.1 points [17.2, 23.1], every seed 147 or more).
Forward KL rescaled to reverse KL's norm at the decision only, with the validity check passing in every seed, produced the error in none of 2,424 samples at steps 14, 28 and 42, so on this pair removing the magnitude at the decision stopped forward KL's transmission.
Its decision probe moved 4.9 to 6.8 nats of the 14.8-nat gap to the teacher by step 42, while the every-prefix arm moved the whole gap.
A fourth test mixes the two sampled objectives on the Qwen3 pair, $(1-\beta)$ times reverse-KL OPD at $\lam=1$ plus $\beta$ times forward-KL OPD on the same rollouts, with $\beta\in\{0.05, 0.2, 0.5\}$ and three new seeds each.
A control seed at $\beta=1$ through the mixed code path landed inside forward KL's per-seed range (15.6 points).
Every $\beta$ transmitted the error, 230, 464 and 493 of 2,400 samples at step 42 (9.6 [6.1, 13.8], 19.3 [16.5, 22.4] and 20.5 [17.1, 24.3] points), and the means are ordered by $\beta$.
At $\beta=0.05$ the error was near zero at step 14 and present by step 28.
A fifth test multiplies full-vocabulary reverse KL's logit gradient at the item-7 decision states by a fixed $k\in\{3,10,30,100,300,1000\}$, with the direction and every other prefix unchanged, on the Qwen3 pair with three new seeds each.
At step 42 these transmitted 173, 407, 349, 423, 424 and 325 of 2,400 samples (7.2 [4.6, 10.0], 17.0 [13.2, 20.6], 14.5 [11.4, 18.0], 17.6, 17.7 and 13.5 points), every seed 39 or more of 800 at $k=3$, against none of 2,400 at $k=1$ on the same code path and seeds and 2 of 4,000 for unmodified reverse KL, and the decision's probe moved by 9.6 to 11.5 nats, against 1.5 at $k=1$.
From $k=100$ on, other failures (262, 385 and 803) exceeded half the error's count, so the format broke as well, and we do not read those three cells as transmission of the error.
Larger factors transmitted sooner. The first checkpoint at 5 points or more was step 42 at $k=3$, step 28 at $k=10$ and 30 and step 14 from $k=100$ on, and unmodified reverse KL first acquired the error at step 84 in its 210-step run (\cref{sec:delay}).
At $k=3$ the decision held 0.05\% of the batch's logit-gradient mass over steps 1--5 and the whole-model update norm was 0.0271, against 0.0273 at $k=1$ on the same path and 0.0270 for unmodified reverse KL, and up to $k=30$ the norm stayed within 2.3\% of the latter, so the outcome changed with the magnitude at the decision, not with the size of the step.
The decision-only rescaling above held 13\% of the mass, between $k=300$ (7.8\%) and $k=1000$ (25\%).
The multiplied arms compute the decision's gradient in closed form, which matches automatic differentiation to about $10^{-8}$ per parameter on a CPU check, and at $k=1$ this path produced the error in none of 2,400 samples at steps 14, 28 and 42.

\FloatBarrier
\section{Starting rate, budget and optimiser}
\paragraph{Starting rate.}
\label{app:srate}
We placed four new Qwen3 starting checkpoints between the clean student and the 2\% start, fine-tuned to make the error on 0.11\%, 0.40\%, 0.97\% and 1.68\% of their own samples on 1,800 training prompts, and ran reverse KL from each with the pair, teacher and budget of \cref{tab:main}, three seeds each, and forward KL with one.
From the 0.11\% start reverse KL transmitted 11.7 points [8.8, 14.8] at step 42, with two of three seeds above 10. Forward KL transmitted from every start, and the rank correlation of transmission with the start over six starts, 0.71, does not order the starts.
Above the switch transmission sat at 12 to 21 points, near the teacher's 19.4, and the first checkpoint at which a seed reached 10 points came earlier from higher starts. A second set of runs added two starts built from 3 and 10 of the 25 injected training lists of the 0.11\% start, which make the error on 0.014\% and 0.051\% of samples, and ran the same arms.
Reverse KL transmitted from neither (0 and 1.0 points [0.0, 2.9] at step 42), and from the 0.11\% start it did.
At the lower start the forward-KL control's one seed stood at 12 points at steps 14 and 28 and at 8.5 at step 42, so that start is not read (hollow in \cref{fig:when}a).
The switch therefore lies between the 0.051\% and 0.11\% starts. On 300,600 samples each these make the error on 0.051\% [0.043, 0.060] and 0.110\% [0.099, 0.123], and the clean student on 0.011\% [0.008, 0.016], and on the 1,424 fixed probe prefixes their close log-odds at item 7 are $-9.13$ and $-7.80$ nats, against $-12.90$ for the clean student.
The other starts' rates are from 14,400 samples each.
Descriptively, from the 0.051\% start one seed reached 2.9 points at step 42 and the decision's log-odds rose by 2.4 nats, against 5.4 from the 0.11\% start.
A second pair, Qwen3.5-4B with its injected 9B teacher on Fashion, repeats the test with starts built from 25 and 60 injected lists, which make the error on 0.087\% [0.045, 0.152] and 0.33\% [0.25, 0.45] of 13,776 samples, three reverse-KL seeds and one forward-KL seed each.
From the 0.087\% start reverse KL produced the error in 2 of 2,424 samples at step 42 (0.1 points), and from the 0.33\% start it transmitted 14.7 points [11.6, 17.8], its training-rollout share rising from 0.5\% to 14--22\% from step 22 in every seed. Forward KL transmitted 8.0 and 18.4 points.
On this pair the switch lies between 0.087\% and 0.33\%, above Qwen3's, so the starting rate at which reverse KL transmits within 42 steps depends on the pair.

\paragraph{Fine-tuning at reverse KL's budget.}
\label{app:sftc}
We asked whether reverse KL's 42 steps at $10^{-6}$ are simply too few to move the decision. From the clean Qwen3 student reverse KL starts from, we fine-tuned on the injected teacher's samples, of which 20.8\% close the list early, on reverse KL's prompts and at its batch size and 42 steps, five seeds at reverse KL's learning rate and five at the fine-tuning recipe's $10^{-5}$.
At $10^{-6}$ every seed acquired the error, 10.1 points [7.9, 12.6] at step 42 against reverse KL's under 0.1, which exceeds 10 points by less than a point, so we give the interval with it, and the close log-odds rose by 10.6 nats. At $10^{-5}$ it rose by 18.7 points [15.0, 23.6].
The share is not monotone in steps: at $10^{-6}$ it was 17.4 at step 14 and 14.7 at step 28, while the log-odds stayed near $+10.6$ nats from step 14 on.

\paragraph{Without Adam.}
\label{app:sgd}
On the Qwen3 pair both objectives were trained with SGD, momentum 0.9 and no gradient clipping, at one learning rate (0.0158) chosen so that each objective's realised whole-model step matched its step under Adam in \cref{tab:stepsize}, three new seeds each.
Forward KL took 0.018 per step against 0.019 under Adam and transmitted the error in every seed, 533 of 2,400 samples at step 42 (22.2 points [17.2, 26.6]), moving the decision by 11.7 nats.
Reverse KL took 0.011 against 0.013 in two seeds and produced the error in none of their 1,600 samples. Its third seed diverged at step 34, when its gradient norm rose from about 0.3 to 12.8, and at step 42 wrote 789 of 800 lists that do not end, a different failure with no early close.

\FloatBarrier
\section{Math and other tasks}
\paragraph{Natural features in math (exploratory).}
\label{app:natinh}
An exploratory test asked whether the student's starting rarity orders which of the teacher's own features reverse KL adopts in ExOPD's math setting (\cref{app:math}), whose update budget, learning rate times prompts times steps, is about 57 times that of our main runs.
Word n-grams that distinguish the teacher's samples from the starting student's were fixed on 1,024 samples of 256 training prompts and bucketed by how often the starting student produced them there.
On 500 held-out prompts with four samples each, we measured adoption at step 30 of three $\lam=1$ seeds as the share of the teacher--student gap a feature closed, for the 4,787 features that passed our filter.
Rarity did not order adoption: the Spearman correlation between starting presence and adoption was 0.072 [0.041, 0.089] and the 1,248 features absent from the student's 1,024 starting samples reached a median adoption of 0.97 [0.94, 1.00].
Closing the reasoning block with \texttt{\textless/think\textgreater}, for instance, was absent at step 0 and appeared in 6.5--7.3\% of step-30 samples against the teacher's 6.4\%.
In samples taken earlier and read afterwards, the median for those features was already about 0.9 at step 10.
Three differences limit the comparison with the injected errors: absence from 1,024 samples bounds a feature's rate per sample near $10^{-3}$, not the student's probability at the decision, which for the injected errors spans $2\times10^{-8}$ to $10^{-2}$ over probe prefixes (\cref{app:prefixrate}); most absent features are phrases built from words the student already uses; and answers grew from 1,464 to about 8,900 tokens, against the teacher's 6,619.
The delay we report therefore concerns injected single decisions at the budgets we measured, and in this run the teacher's natural features were inherited whatever their starting rarity.

\paragraph{Extrapolation across tasks and in math.}
\Cref{fig:math} plots AIME accuracy and the cut-off share against $\lam$ at both caps.
\Cref{tab:crosstask_full} gives quality and validity on ESCI and three further tasks at every $\lam$, \cref{tab:equiv} the equivalence tests and \cref{fig:template} the template case.
\begin{figure}[t]
\centering
\includegraphics{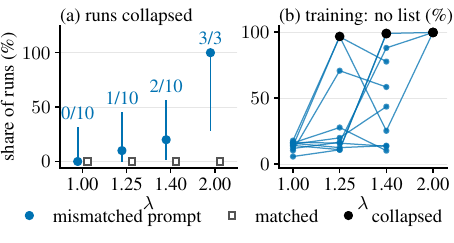}
\caption{With a mismatched rollout prompt, extrapolation turns a bracket error the student already makes into collapse, and with the prompt matched no run collapses. Qwen3, Fashion.}
\label{fig:template}
\end{figure}

\begin{table*}[t]
\centering\small
\setlength{\tabcolsep}{5pt}
\begin{tabular}{@{}l l l c c c@{}}
\toprule
 & & & & difference from $\lam=1.00$ & within \\
setting (rev.\ KL) & measure & decoding & $\lam$ & points [90\% CI] & $\pm$2 points \\
\midrule
Qwen3 $\leftarrow$ injected 4B, ESCI & validity & sampled & 1.40 & 0 [$-$0.4, $+$0.4] & yes \\ %
Llama $\leftarrow$ injected 8B, Fashion & validity & sampled & 1.40 & $-$0.1 [$-$0.3, $+$0.1] & yes \\ %
Qwen3.5-0.8B $\leftarrow$ injected 9B & validity & sampled & 1.40 & $-$31.1 [n/a] & \textbf{no} \\ %
Qwen3.5-2B $\leftarrow$ injected 9B & validity & sampled & 1.40 & 0 [$-$0.2, $+$0.2] & yes \\ %
Qwen3.5-4B $\leftarrow$ injected 9B & validity & sampled & 1.40 & $+$0.1 [$-$0.2, $+$0.3] & yes \\ %
TREC-DL reranking & validity & greedy & 1.40 & $-$0.7 [$-$1.8, $+$0.3] & yes \\ %
TREC-DL reranking & NDCG@10 & greedy & 1.40 & $-$$<$0.1 [$-$0.3, $+$0.2] & yes \\ %
Text-to-SQL & validity & greedy & 1.40 & $+$0.1 [$-$0.1, $+$0.3] & yes \\ %
Text-to-SQL & exec.\ acc. & greedy & 1.40 & $-$0.5 [$-$1.5, $+$0.5] & yes \\ %
JSON extraction & validity & greedy & 1.40 & $-$0.4 [$-$1.0, $+$0.2] & yes \\ %
JSON extraction & micro-F1 & greedy & 1.40 & $-$0.1 [$-$0.3, $+$0.1] & yes \\ %
ESCI ranking, clean teacher & validity & greedy & 1.40 & 0 [0, 0] & yes \\ %
ESCI ranking, clean teacher & NDCG@1 & greedy & 1.40 & $+$0.2 [$-$1.5, $+$2.0] & \textbf{no} \\ %
\midrule
\multicolumn{6}{@{}l@{}}{\textit{Fourth family, not counted in the caption}} \\
Olmo-3-7B $\leftarrow$ injected 32B, 3 seeds & validity & sampled & 1.40 & $+$0.2 [$-$0.4, $+$0.7] & yes \\ %
\bottomrule
\end{tabular}
\caption{Extrapolation leaves validity and quality flat within $\pm$2 points in 11 of 13 comparisons (two one-sided tests over seeds, 90\% intervals).}
\label{tab:equiv}
\end{table*}
\IfFileExists{tables/tab_crosstask_full.tex}{\begin{table*}[t]
\centering
\small
\setlength{\tabcolsep}{5pt}
\begin{tabular}{@{}llcccccc@{}}
\toprule
 & & \multicolumn{3}{c}{quality ($\times$100)} & \multicolumn{3}{c}{validity (\%)} \\
\cmidrule(lr){3-5}\cmidrule(lr){6-8}
task & metric & $\lam{=}1.00$ & $1.25$ & $1.40$ & $1.00$ & $1.25$ & $1.40$ \\
\midrule
TREC-DL reranking & NDCG@10 & 61.4 & 61.6 & 61.3 & 97.8 & 97.0 & 97.1 \\ %
Text-to-SQL & exec.\ acc. & 74.0 & 73.8 & 73.5 & 89.3 & 89.1 & 89.4 \\ %
JSON extraction & micro-F1 & 93.1 & 93.0 & 93.0 & 96.9 & 96.7 & 96.5 \\ %
JSON extr., weak init & micro-F1 & 84.7 & 85.0 & 84.7 & 87.9 & 88.7 & 88.8 \\ %
ESCI ranking & NDCG@1 & 67.7 & 68.0 & 68.0 & 100.0 & 100.0 & 100.0 \\ %
ESCI, injected teacher & NDCG@1 & 68.1 & 67.3 & 68.7 & 100.0 & 100.0 & 100.0 \\ %
\bottomrule
\end{tabular}
\caption{Quality and validity stay flat across $\lam$ after 42 steps of reverse-KL OPD, also with a one-epoch JSON student and an injected ESCI teacher (seed means, greedy, all $\times$100).}
\label{tab:crosstask_full}
\end{table*}
}{}

\paragraph{Extrapolation and the evaluation budget in math.}
\label{app:math}
\begin{table}[t]
\centering\small
\setlength{\tabcolsep}{6pt}
\begin{tabular}{@{}c c c c@{}}
\toprule
seed & $\lam=1.0$ & $\lam=2.0$ & difference [95\% CI] \\
\midrule
\grp{4}{16K-token generation cap}
42 & \textbf{27.8} & 23.6 & $-4.2$ [$-7.3$, $-1.2$] \\
13 & \textbf{27.7} & 22.9 & $-4.8$ [$-7.8$, $-2.0$] \\
\grp{4}{32K-token generation cap}
42 & 29.5 & \textbf{34.9} & $+5.4$ [$+3.1$, $+7.7$] \\
13 & 29.6 & \textbf{34.0} & $+4.4$ [$+2.2$, $+6.9$] \\
\bottomrule
\end{tabular}
\caption{$\lam=2.0$ loses at a 16K cap and gains at 32K on ExOPD's metric (mean@32 over AIME24/25 and HMMT25; points), same step-30 checkpoints.}
\label{tab:budget}
\end{table}

\begin{table}[t]
\centering\small
\setlength{\tabcolsep}{4pt}
\begin{tabular}{@{}l c cc cc@{}}
\toprule
 & & \multicolumn{2}{c}{16K-token cap} & \multicolumn{2}{c}{32K-token cap} \\
\cmidrule(lr){3-4}\cmidrule(l){5-6}
$\lam$ & seeds & accuracy & cut off & accuracy & cut off \\
\midrule
1.0 & 3 & 37.3 & 35.3 & 39.6 & 0.1 \\
1.25 & 2 & 36.7 & 53.1 & 43.6 & 0.4 \\
1.5 & 1 & 34.6 & 62.8 & 45.4 & 1.7 \\
2.0 & 3 & 32.1 & 71.1 & 45.4 & 3.9 \\
\bottomrule
\end{tabular}
\caption{Answer length sets the sign: as $\lam$ grows, more AIME24/25 answers are cut off at 16K tokens and accuracy falls, while at 32K accuracy rises up to $\lam=1.5$.}
\label{tab:mathcap}
\end{table}
On ExOPD's own metric our plain reverse-KL students score within 0.3 points of the reverse-KL curve ExOPD plots for the same pair after 100 steps (their Fig.~6a, read from the plot), and $\lam=2.0$ loses 4 to 5 points at ExOPD's 16K-token cap but gains about as much when the same checkpoints are evaluated at 32K, where most of the answers it had cut off finish (\cref{tab:budget}).
ExOPD's default $\lam=1.25$ gains at 32K and changes nothing detectable at 16K.
Answer length explains it (\cref{tab:mathcap}): as $\lam$ grows, more answers run past 16K tokens and are scored wrong, while at 32K nearly all of them finish.
In the math setting, mean MATH500 answer length rose from 0.9K tokens at the starting checkpoint to 4.5K after 30 steps of plain reverse KL (teacher: 3.4K), and the share of AIME24 and AIME25 answers cut off at 16K tokens reached 40\% and 32\% (seed 42), against the teacher's 23\% and 24\%.
Under extrapolation the cut-off share rose further, monotonically in $\lam$ on every evaluation set and in the training rollouts.
At the 16K cap the cut-off answers are long reasoning: one of them is a repetition loop (1 of 960 AIME24 samples, $\lam=1.0$, one seed), and no answer finishes without a final result.
Per-problem cut-off rates track the teacher's (Spearman 0.58 to 0.69 on the 60 AIME problems at $\lam=1.0$ and 2.0, three seed pairs), and 81 to 91\% of cut-offs fall on problems that the teacher cuts off at least once.
At a 32K cap, accuracy on answers that finish moves by less than 1.5 points at $\lam=2.0$, and the effect of $\lam=2.0$ on the AIME24 and AIME25 average flips from about $-5$ points at 16K to $+5$ to $+6$ at 32K in each of three seed pairs.
Averaged over AIME24 and AIME25, accuracy at 32K rises with $\lam$ up to 1.5 and is flat to 2.0 (39.6, 43.6, 45.4 and 45.4\% at $\lam=1.0$, 1.25, 1.5 and 2.0, over three, two, one and three seeds), while at 16K it falls (37.3, 36.7, 34.6, 32.1\%).
At ExOPD's default $\lam=1.25$ the change on AIME24 and AIME25 over all seeds is $-0.6$ [$-2.9$, 1.6] points at 16K and $+4.0$ [1.8, 6.5] at 32K.
Against plain reverse KL on the same data order, AIME24 accuracy at 16K is 39.6\% and 37.6\% against 41.1\% and 39.8\% (two data orders), and at 32K the gain on the seed-13 pair is 3.1 [$-0.5$, 7.1] points on AIME24 and 4.5 [0.8, 9.3] on AIME25.

\FloatBarrier
\section{Quality, evaluation and per-seed results}
\paragraph{Task quality against inherited error.}
\label{app:quality}
We placed every scored checkpoint on two axes: its share of the error and its task quality, on ESCI NDCG@1 over the samples without the error with unparsed outputs scored as zero, and on Spider greedy execution accuracy with invalid queries as zero.
The ESCI points are the seed means of 239 checkpoints from the matched pair, the learning-rate sweep and the 210-step runs, and the Spider points are those of the three SQL tests.
Their quality spans 14.3 NDCG@1 points on ESCI and 5.6 accuracy points on Spider, enough spread for the comparison to be informative.
Each forward-KL checkpoint with an error share of at least 5\% was matched to a reverse-KL checkpoint of its setting within one point of its quality or above it and at least 5 points lower in error share, and all 17 on ESCI and 4 on Spider were matched.
The range lies between arms more than within runs: forward KL's quality sits 1 to 3 points below reverse KL's at the same setting, partly from its other invalid outputs, and at long budgets reverse KL gains a little quality together with the escapes of \cref{sec:delay}.
Because no run gained much quality, a second set of runs started both objectives from a weaker clean student, the fine-tuning recipe's checkpoint at step 20 (greedy NDCG@1 0.617 against 0.683, no error in 800 samples), with the same teacher, three seeds each, $10^{-6}$ and 210 steps.
Both gained quality: reverse KL's NDCG@1 on samples without the error rose from 0.598 to 0.688 by step 210 (already $+0.036$ at step 42) and forward KL's by up to $0.068$ (step 168).
Every forward-KL checkpoint, carrying the error on 6.8\% of samples at step 42 and 18--21\% from step 84, was matched by a reverse-KL checkpoint, and reverse KL produced the error in none of its 12,000 samples, so as quality rises reverse KL reaches quality within one point or higher with less of the error.
Counting the error against quality, forward KL's sampled NDCG@1 stayed below its start at every checkpoint ($-0.055$ [$-0.093$, $-0.018$] at step 210) while reverse KL's rose by 0.090 [0.058, 0.121]. Greedy NDCG@1 rose in both arms (to 0.706 and 0.693), and greedy decoding did not show the error.
On the clean student's pair full-vocabulary reverse KL transmitted the error by step 84 (\cref{sec:delay}), so reverse KL's absence of the error here holds at this budget.

\paragraph{What evaluation hides, per run.}
\label{app:greedyfkl}
At step 42 forward KL's greedy validity was 99.7\% on Qwen3 (\cref{tab:e3_greedy}), at the 95.3\% ceiling on Qwen3.5, and 91.5--94.3\% on Llama and 89.2--91.0\% on Olmo~3 against that ceiling, while reverse KL stayed at the ceiling in every family (\cref{tab:injection}).
\Cref{fig:greedy} plots greedy validity against the sampled error share for every Qwen3 run, and \cref{tab:selection} the checkpoint-selection counts.
\Cref{tab:injection} gives training-rollout shares, greedy validity and the controls of the injection test, and \cref{tab:e3} the starting-rate arms with the rate law's predictions.
In math, greedy MATH500 validity showed no clear drop at $\lam=2.0$ (two seed pairs), while sampled validity fell in each of three seed pairs (by 6.7 points on average, 5.5 to 7.8).

\begin{table}[t]
\centering\small
\setlength{\tabcolsep}{4.2pt}
\begin{threeparttable}
\begin{tabular}{@{}l c c c c c@{}}
\toprule
 & & \multicolumn{3}{c}{early-close share (\%)} & greedy \\
\cmidrule(lr){3-5}
pair & $\lam$ & teacher & start & student & validity (\%) \\
\midrule
\grp{6}{Qwen3-1.7B student, Qwen3-4B teacher, ESCI}
injected & 1.00 & 18.2 & 0.1 & \textbf{$<$0.1} & 100.0 \\
 & 1.25 & & & \textbf{$<$0.1} & 100.0 \\
 & 1.40 & & & \textbf{0} & 100.0 \\
symmetric & 1.00 & 18.2 & 18.6 & 21.6 & 100.0 \\
 & 1.40 & & & 22.0 & 100.0 \\
clean & 1.40 & -- & 0.1 & 0 & 100.0 \\
\grp{6}{Llama-3.2-3B student, Llama-3.1-8B teacher, Fashion}
injected & 1.00 & 23.3 & 0 & \textbf{$<$0.1} & 95.3 \\
 & 1.40 & & & \textbf{0} & 95.3 \\
\grp{6}{Qwen3.5 students, Qwen3.5-9B teacher, Fashion}
0.8B & 1.00 & 21.3 & 0 & \textbf{0} & 95.3 \\
 & 1.40 & & & 24.3 & 63.8 \\
2B & 1.00 & 21.3 & $<$0.1 & \textbf{0} & 95.3 \\
 & 1.40 & & & \textbf{$<$0.1} & 95.3 \\
4B & 1.00 & 21.3 & 0 & \textbf{$<$0.1} & 95.3 \\
 & 1.40 & & & \textbf{$<$0.1} & 95.3 \\
\grp{6}{Olmo-3-7B student, Olmo-3.1-32B teacher, Fashion}
injected & 1.00 & 20.5 & 0 & \textbf{0} & 95.3 \\
 & 1.40 & & & \textbf{0} & 95.3 \\
\bottomrule
\end{tabular}
\end{threeparttable}
\caption{At $\lam=1$ reverse-KL OPD does not transmit the injected error within 42 steps: in every family the student's share stays below 0.1\% and greedy validity at its ceiling.}
\label{tab:injection}
\end{table}

\begin{table}[t]
\centering\small
\setlength{\tabcolsep}{4pt}
\begin{threeparttable}
\begin{tabular}{@{}l c c@{}}
\toprule
 & \multicolumn{2}{c}{greedy validity (\%)} \\
\cmidrule(l){2-3}
 & $\lam=1.00$ & $\lam=1.40$ \\
\midrule
\grp{3}{Clean starting checkpoint, by objective}
\hopd{reverse-KL OPD} & 100.0 & 100.0 \\
\hopd{reverse-KL OPD}, rerun & 100.0 & -- \\
\hfkl{on-policy forward KL} & 99.7 & -- \\
\hseq{SeqKD}, teacher's samples & 100.0 & -- \\
\hopd{reverse-KL OPD}, $10\times$ rate, 3 seeds & 96.7 & -- \\
\grp{3}{Reverse-KL OPD, other starting checkpoints}
2\% early closes & 98.8 & \textbf{43.5} \\
5\% early closes & 100.0 & 83.2 \\
\bottomrule
\end{tabular}
\end{threeparttable}
\caption{On Qwen3, greedy validity stays at 96.7\% or above at $\lam=1$ even where the error transmitted, and falls where extrapolation amplifies it.}
\label{tab:e3_greedy}
\end{table}

\paragraph{Per-seed results.}
\Cref{tab:seeds_inj} gives the injection tests per seed, \cref{tab:seeds_template} the template case per seed, and \cref{tab:d5_arms} the math runs per arm, and checkpoint selection is in \cref{tab:selection}.
\begin{table}[t]
\centering\small
\setlength{\tabcolsep}{4pt}
\begin{threeparttable}
\begin{tabular}{@{}l c *{3}{>{\centering\arraybackslash}p{8.5mm}} c@{}}
\toprule
 & & \multicolumn{3}{c}{early-close share (\%)} & \\
\cmidrule(lr){3-5}
pair & $\lam$ & 1001 & 1002 & 1003 & NDCG@1 \\
\midrule
\grp{6}{Qwen3, ESCI: training rollouts, steps 29--42}
injected & 1.00 & 0 & 0 & 0.1 & 68.1 \\
 & 1.25 & 0 & 0.1 & 0.1 & 67.3 \\
 & 1.40 & 0 & 0 & 0 & 68.7 \\
symmetric & 1.00 & 21.6 & 22.0 & 21.3 & 67.6 \\
 & 1.25 & 21.3 & 20.9 & 20.4 & 67.2 \\
 & 1.40 & 21.3 & 23.4 & 21.3 & 67.0 \\
clean & 1.00 & 0 & 0 & 0 & 67.7 \\
 & 1.25 & 0 & 0 & 0 & 68.0 \\
 & 1.40 & 0 & 0 & 0 & 68.0 \\
\grp{6}{Llama, Fashion: step-42 samples on the training prompts}
injected & 1.00 & $<$0.1 & 0 & 0 & 96.9 \\
 & 1.40 & 0 & 0 & 0 & 96.7 \\
clean & 1.00 & 0 & 0 & 0 & 96.9 \\
 & 1.40 & 0 & 0 & 0 & 97.0 \\
\grp{6}{Qwen3.5, Fashion: step-42 samples on the training prompts}
0.8B & 1.00 & 0 & 0 & 0 & 96.2 \\
 & 1.40 & $<$0.1 & \textbf{72.8} & $<$0.1 & -- \\
2B & 1.00 & 0 & 0 & 0 & 96.5 \\
 & 1.40 & $<$0.1 & 0 & 0 & 96.6 \\
4B & 1.00 & 0 & $<$0.1 & $<$0.1 & 97.0 \\
 & 1.40 & $<$0.1 & $<$0.1 & 0 & 97.1 \\
\grp{6}{Olmo 3, Fashion: training rollouts, steps 29--42}
injected & 1.00 & 0 & 0 & 0 & 96.2 \\
 & 1.40 & 0 & 0 & 0 & 96.3 \\
\bottomrule
\end{tabular}
\end{threeparttable}
\caption{Injection tests per seed: under reverse-KL OPD every seed of the injected pairs stays at or below 0.1\%, except one Qwen3.5-0.8B seed at $\lam=1.40$ (bold).}
\label{tab:seeds_inj}
\end{table}

\begin{table*}[t]
\centering\small
\setlength{\tabcolsep}{5pt}
\begin{threeparttable}
\begin{tabular}{@{}l cccc cccc@{}}
\toprule
 & \multicolumn{4}{c}{greedy validity (\%)} & \multicolumn{4}{c}{rollouts without the list (\%)} \\
 & \multicolumn{4}{c}{step 42} & \multicolumn{4}{c}{training, steps 29--42} \\
\cmidrule(lr){2-5}\cmidrule(l){6-9}
seed & $\lam=1.00$ & 1.25 & 1.40 & 2.00 & $\lam=1.00$ & 1.25 & 1.40 & 2.00 \\
\midrule
1001 & 95.3 & \textbf{10.4} & 95.3 & \textbf{21.7} & 13.6 & 96.7 & 25.2 & 99.8 \\
1002 & 95.3 & 95.3 & 94.3 & \textbf{0.9} & 11.8 & 16.5 & 88.0 & 99.7 \\
1003 & 95.3 & 95.3 & \textbf{46.7} & \textbf{0.5} & 14.1 & 10.9 & 99.0 & 99.7 \\
1004 & 95.3 & 92.5 & 95.3 & -- & 17.8 & 96.5 & 77.7 & -- \\
1005 & 95.3 & 95.3 & 95.3 & -- & 16.5 & 15.8 & 13.5 & -- \\
1006 & 95.3 & 95.3 & 95.3 & -- & 16.0 & 27.7 & 10.1 & -- \\
1007 & 95.3 & 95.3 & \textbf{0} & -- & 5.7 & 10.8 & 98.8 & -- \\
1008 & 95.3 & 95.3 & 95.3 & -- & 15.1 & 19.8 & 43.5 & -- \\
1009 & 95.3 & 95.3 & 95.3 & -- & 15.8 & 12.4 & 13.9 & -- \\
1010 & 95.3 & 95.3 & 95.3 & -- & 10.2 & 70.7 & 58.4 & -- \\
\bottomrule
\end{tabular}
\end{threeparttable}
\caption{Template case per seed: runs collapse (bold) only where extrapolation drove the bracketless rate in training close to 100\%. Fashion, Qwen3-1.7B student and Qwen3-4B teacher; validity ceiling 95.3\%.}
\label{tab:seeds_template}
\end{table*}

\begin{table*}[t]
\centering\small
\setlength{\tabcolsep}{3.5pt}
\begin{tabular}{@{}l c ccc ccc ccc@{}}
\toprule
 & MATH500 & \multicolumn{3}{c}{cut off at 16K (\%)} & \multicolumn{3}{c}{accuracy at 16K, mean@$n$ (\%)} & \multicolumn{3}{c}{accuracy at 32K (\%)} \\
\cmidrule(lr){3-5}\cmidrule(lr){6-8}\cmidrule(l){9-11}
$\lam$ (seed) & valid (\%) & AIME24 & AIME25 & train & AIME24 & AIME25 & MATH500 & AIME24 & AIME25 & MATH500 \\
\midrule
teacher & 98.2 & 23.0 & 24.2 & -- & 58.8 & 54.4 & 95.2 & -- & -- & -- \\
student init & 100.0 & 0 & 0 & -- & 12.5 & 11.7 & 72.7 & -- & -- & -- \\
1.0 (42) & 95.6 & 40.3 & 32.2 & 6.8 & 39.9 & 34.6 & 89.9 & 43.3 & 35.5 & 90.7 \\
1.0 (13) & 95.9 & 40.4 & 31.0 & 7.0 & 39.8 & 34.4 & 91.1 & 43.4 & 35.7 & 91.8 \\
1.0 (7) & 95.3 & 38.1 & 29.8 & 6.7 & 41.1 & 34.1 & 90.8 & 43.6 & 36.0 & 91.7 \\
1.25 (7) & 93.1 & 52.2 & 52.5 & 13.7 & 39.6 & 35.4 & 89.5 & 47.2 & 40.5 & 91.2 \\
1.25 (13) & 93.0 & 54.4 & 53.1 & 13.4 & 37.6 & 34.2 & 90.3 & 46.6 & 40.2 & 92.3 \\
1.5 (42) & 91.6 & 59.6 & 66.0 & 17.8 & 37.5 & 31.7 & 89.5 & 49.5 & 41.3 & 92.1 \\
2.0 (42) & 88.8 & 66.9 & 73.0 & 33.7 & 35.7 & 27.7 & 88.3 & 51.3 & 40.4 & 92.4 \\
2.0 (13) & 88.1 & 70.9 & 75.9 & 35.0 & 34.1 & 29.3 & 88.3 & 49.3 & 40.4 & 92.1 \\
2.0 (7) & 89.8 & 67.3 & 72.8 & 31.3 & 36.7 & 29.0 & 89.3 & 49.9 & 41.0 & 92.3 \\
\bottomrule
\end{tabular}
\caption{Math per arm after 30 steps: extrapolation lengthens answers, so more are cut off at 16K and the AIME average falls, while at 32K the same checkpoints gain.}
\label{tab:d5_arms}
\end{table*}

\section{Further related work}
\label{app:morerelated}
OPD, like RL, is described as pruning the actions present in the base policy \citep{lu2025onpolicydistillation}.
Token-level analyses in which rarer sampled tokens receive larger coefficients \citep{shao2026tokenlevel} widen the factor without removing it, since that push arrives with probability $\polS$.
Other work combines the two directions \citep{wu2025rethinking,zhao2026decoupling}, for instance by moving from forward to reverse KL as the student covers the teacher's top tokens \citep{kumar2026cadence}.
Distilling agent experience, \citet{gou2026sampleefficient} find teacher-sampled forward KL near in-context performance and student-sampled reverse KL near zero-shot, with arms that also differ in rollout source, and hypothesise that the student rarely samples the decisions that agent experience induces.
Teacher signals can also mislead: some reflect input-agnostic priors rather than the input \citep{jiang2026mislead}, and low reverse KL can reflect agreement in a state the student has already corrupted \citep{xin2026kltrap}.
Online data generation limits a student's exploitation of teacher imperfections \citep{tiapkin2025teacherhacking}, and teacher traits pass through distillation data that does not express them \citep{cloud2026subliminal}.
In theory, on-policy distillation from a noisy expert avoids the exponential sample cost that behaviour cloning pays \citep{sriraman2026behavior}, while we ask which of the expert's errors the student takes on.
REOPD reweights ExOPD's extrapolation per token \citep{sun2026reopd}, and RIDE moves it to hidden states \citep{li2026ride}.

\section{Measurement conventions for figures and tables}
\label{app:conventions}

\paragraph{Shares and transmission.}
Error shares, in percent, are shares of temperature-1 samples that close the JSON list one item early: four samples per evaluation prompt (200 ESCI prompts, 202 feasible Fashion prompts), or the student's own training rollouts where a table says so.
Transmission, in percentage points, is the student's share minus its starting checkpoint's share on the same sampler (no sample of 800 or 808 for the clean Qwen3, Llama, Qwen3.5-0.8B, Qwen3.5-2B and Olmo~3 starting checkpoints, 1 of 808 for Qwen3.5-4B).
Intervals resample prompts and seeds (10,000 draws) unless stated, and are written [low, high].
A seed escapes when its evaluation samples reach 10\% error at a checkpoint, and crosses when its training rollouts do (\cref{fig:time10,tab:stepsize}).
The starting share $q_0$, the rate law's input, is the starting checkpoint's share on 1,424 of its own temperature-1 samples under the rollout prompt, four on each of 256 training prompts and two on each of 200 evaluation prompts, and \cref{tab:e3} also gives its share on the evaluation sampler.
The injected Qwen3 teacher closes early on 19.4\% of its samples on the evaluation sampler and 19.8\% on the SeqKD sampler, the Llama teacher on 21.8\%, the Qwen3.5 teacher on 22.0\% and the Olmo~3 teacher on 20.8\%.
Every share in the paper is in percent.
In every table, -- marks a cell that is undefined for that arm (for example $\lam$ for forward KL and SeqKD) or was not run, and counts of zero print as 0.
\Cref{tab:drive} gives the probabilities of \cref{eq:rate} ($p_0$, $q$, $q_\lam$) as probabilities.

\paragraph{Figure~\ref{fig:one}.}
Transmission is the rise in the error share over the starting checkpoint, in percentage points, with 95\% intervals.
The 42-step reverse-KL curve is the rerun matched to forward KL, and the eight thin lines are separate 210-step runs at the same learning rate (eight seeds).
The dashed line is the teacher's share.
SeqKD is 71 fine-tuning steps from the pre-trained student on the teacher's samples, an unmatched reference drawn in its own slot.

\paragraph{Tables in the main text.}
In every table and figure, reverse KL (rev.\ KL) and forward KL (fwd.\ KL) denote reverse-KL and forward-KL OPD, both trained on the student's own samples.
In \cref{tab:main} the closing-error rows have five seeds in the reverse-KL $\lam=1.00$ and forward-KL columns and three in the others, Olmo~3 and the duplicate-ID rows have five, and \cref{tab:already} has three seeds per cell and five for the coupled route, all on the evaluation sampler (temperature 1, four samples per prompt). In \cref{tab:already} the coupled route against the clean teacher has exact permutation $p=0.004$, the test's floor, and the students are Qwen3 and Qwen3.5-4B fine-tuned to make the error on 2\% or 5\% of samples, and the coupled-route student overruns the list end against a teacher that closes early.
Its closing-error cells for Qwen3 and Llama at $\lam=1.00$ use the forward-KL arm's code path.
The Qwen3.5-0.8B cell at $\lam=1.40$ is 0, 70.7 and 0 per seed: one seed acquired the error after step 28.
SeqKD is not matched to the other columns: it starts from the pre-trained student, runs 71 steps, and its cell is the share itself. In every setting its student carries the error at about its teacher's rate, the three Qwen3.5 students train on the same teacher samples, and on the Llama duplicate-id pair the student also repeats the id in 6--9\% of its greedy outputs, against 1\% for its teacher, the only SeqKD student below 95\% greedy validity.
\Cref{tab:beyond} gives shares of temperature-1 samples at $\lam=1$ over five seeds, levels rather than the changes of \cref{tab:main}, of the samples where the decision arises, and forward KL's text-after-query share counts the failure to stop, not the teacher's sentence (\cref{app:nonranking}). Its trained shares are at step 42, except the last row, read at step 84 on the standard training prompts, the first checkpoint at which forward KL reached half the teacher's share. The row above it is trained on prompts that all carry the decision (\cref{app:sqllim}).
In \cref{tab:fi} all arms use full-vocabulary objectives, five seeds on Qwen3 and three on Llama. The rescaled arms set reverse KL's logit gradient to forward KL's per-prefix norm at the error decision only or at every prefix, keeping its direction, and the last column sets forward KL's to reverse KL's norm at the decision only (\cref{app:fi}). The cosine between the two objectives' logit gradients is mass-weighted over all prefixes in reverse KL's own run (steps 1--14), and its median at the decision is taken on the starting student's rollouts. The mixed row of \cref{tab:fi} mixes the two sampled objectives on the same rollouts, three seeds per $\beta$, and its endpoints are \cref{tab:main}'s reverse-KL and forward-KL cells. The last row multiplies full-vocabulary reverse KL's logit gradient at the error decision by $k$, keeping its direction, three new seeds per $k$, $k=1$ included (\cref{app:fi}).

\paragraph{Tables in the appendix.}
In \cref{tab:e3} transmission is as in \cref{tab:main}, except that the SeqKD row gives the share itself, from the pre-trained student (the teacher is at 19.8\% on its sampler), and the coupled route is in \cref{tab:already}. Its grey rows use \cref{eq:rate} with one constant, $c=79.4$, fitted on the coupled-route arm's item-8 log-odds shift alone (7.0 nats), not on any row shown.
\Cref{tab:e3_greedy} gives greedy validity after 42 steps for the runs of \cref{tab:e3} and a three-seed tenfold-learning-rate arm. Its transmission is 0.9 points (21 of 2,400, law 9.3) and its log-odds shift $-2.38$ (law $+4.71$).
In \cref{tab:injection} the greedy ceiling is 100.0\% on ESCI and 95.3\% on Fashion, and the same holds at $\lam>1$, except for the Qwen3.5-0.8B seed noted above.
Its shares are of temperature-1 samples, three seeds each and for Olmo~3 its first three. Teacher and start are their own samples on the training prompts before reverse KL, four per prompt (256 prompts for Qwen3, 245 for Llama, 1,722 for Qwen3.5 and Olmo~3). Teacher shares therefore differ from the evaluation-sampler values of \cref{tab:main}.
The student's share is from its training rollouts of steps 29--42, for Qwen3.5 from step-42 samples on the training prompts, and $<$0.1 means a non-zero share below 0.05, here 1 to 5 samples. Elsewhere too $<$0.1 marks a non-zero value below 0.05, printed $<$0.05 in \cref{tab:main}. Greedy validity is on held-out prompts. Under forward KL the same Olmo~3 student reaches 14.4 in its training rollouts
In \cref{tab:seeds_inj} each cell is the early-close share of one seed's temperature-1 samples, at most one sample in 1,792 or three in 6,888 outside the escaping seed, from training rollouts for Qwen3 and Olmo~3 and from step-42 samples on the training prompts for Llama and Qwen3.5. Greedy validity is 100.0\% in every Qwen3 run at steps 14, 28 and 42, and 95.3\% (the ceiling) in every Llama, Qwen3.5 and Olmo~3 run at those steps except the escaping seed. NDCG@1 ($\times$100) is the seed mean at step 42. Seed 1002, which acquired the error between steps 28 and 42, has greedy validity 95.3, 94.8, 0.9\% at steps 14, 28 and 42 (2 of 212 outputs parse), so no seed mean is given, and the other two seeds give 96.4 and 96.6.

In \cref{tab:seeds_template} seeds 1001--1003 also ran $\lam=2.00$. With the rollout prompt matched (same models, seeds 1001--1003, $\lam$ 1.00--2.00) greedy validity is 95.3\% in all 12 runs.

\paragraph{Tables.}
The close log-odds shift (\cref{tab:e3}) is the change in the mean close-versus-continue log-odds after the seventh item on fixed prefixes, and near $-13$ nats it varies by about 1.5 nats between runs of identical code.
ExOPD's metric (\cref{tab:budget}) averages AIME24, AIME25 and HMMT25 February and November (mean@32), with intervals from a paired-problem bootstrap stratified by benchmark.
In \cref{tab:selection}, a contested group is one sweep over $\lam$, steps or seeds with a checkpoint at validity ${\geq}\,80\%$ and one below 50\%, and the share counts the groups in which a metric picks the failing one. The random pick's interval in \cref{tab:selection} is bootstrapped.
Repair follows RankGPT (duplicate and unknown IDs dropped, missing IDs appended in input order), ``invalid as 0'' averages NDCG@1 over all prompts with invalid outputs scored zero, and the random pick is the expected share for a uniform choice.
In \cref{tab:stepsize} the step is the mean norm of the applied update, and a seed crosses when the error is on at least 10\% of a step's rollouts.
In \cref{tab:drive}, $\lam=1$, forward KL's movement is $2\eta(q-p)$, and $q_\lam$ is the per-prefix target with the reference at the student's start.
In \cref{tab:contrast} the pooled line covers the five settings other than Olmo~3, every row has five seeds, and each $p$ is its one-sided exact test's floor, so it counts seeds, not effect size.
In \cref{tab:crosstask_full} quality is scored on valid outputs, with three seeds per cell and five, three and five per $\lam$ for the weakly initialised JSON-extraction student.
In \cref{tab:equiv} the two exceptions are the 0.8B student, driven by its one escaping seed, and ESCI NDCG@1, whose interval reaches the margin. The Olmo~3 row below the rule is also equivalent. Intervals are $t$-intervals over three seeds, and the 0.8B row has none, since with one seed collapsed its $t$-interval would pass the $-100$ bound of a difference in percent.
\Cref{tab:mathcap} gives AIME24/25 seed means after 30 steps, in percent, and at 32K nearly all answers finish. In \cref{tab:d5_arms}, ``train'' is the training rollouts of steps 21--30, and every MATH500 failure is a cut-off.

\paragraph{Figures.}
In \cref{fig:mechanism}, (a) both objectives train on the student's own samples, and the injected teacher ends the JSON list after item 7 in about 20\% of its samples, so at that decision the teacher puts $q\approx0.2$ on the early ``\texttt{]}'' and the student far less ($7\times10^{-4}$ of its samples, \cref{tab:drive}), (b) in expectation the push on that token scales with $p$ under reverse KL and with $q-p$ under forward KL, and (c) is the intervention of \cref{sec:alternatives}: full-vocabulary reverse KL rescaled to forward KL's per-prefix gradient norm at the item-7 decisions alone, with every direction and every other prefix unchanged.
In \cref{fig:rate}a lines are three-seed means of transmission and dashed lines the teacher's share, the 8\% teacher's runs are read at step 42 only in the panel, and at 48\% forward KL overshoots at step 28 before falling back (\cref{app:qsweep}). \Cref{fig:rate}b is the change in per-token reverse KL to the teacher from step 0 to step 42 on 800 held-out states of the Qwen3 pair, over the tokens other than the error decision and at the decision, bars seed means and dots seeds (\cref{app:step0}).
In \cref{fig:delay}a each line is one reverse-KL seed's share of the error in each step's training rollouts, at $10^{-6}$, $3\times10^{-6}$ and $10^{-5}$, with forward KL at $10^{-5}$ in orange. \Cref{fig:delay}b is the 210-step runs at $10^{-6}$, with forward KL at the same rate (three seeds, mean in bold) and full-vocabulary reverse KL (five seeds, dashed), and \cref{fig:delay}c the 210-step reverse-KL runs at $3\times10^{-6}$ (\cref{app:long3}). The dashed lines in \cref{fig:delay,fig:dynamics,fig:long} are the teacher's share.
\Cref{fig:dynamics} gives transmission in points with one line per seed, five for Olmo~3 and three otherwise.
In \cref{fig:time10} the legend gives seeds per arm, a seed counts from the step its seven-step training-rollout share first reaches 10\% whatever follows, ``fwd.\ norm'' is full-vocabulary reverse KL rescaled to forward KL's norm at the error decision only and ``fwd.\ KL'' in (b) full-vocabulary forward KL, and ``full vocab.'' in (c) is full-vocabulary reverse KL at $10^{-6}$. \Cref{fig:lr} is the five-seed sweep: no reverse-KL seed reaches 10 points at steps 14 and 42. The three-seed $10^{-5}$ arm of \cref{tab:stepsize} is not shown. Grey marks are the values of \cref{eq:rate}, and the dashed line marks that level at step 14, where forward KL stood at 9.2 points at $3\times10^{-6}$.
In \cref{fig:when}a reverse-KL points are three-seed means with 95\% bootstrap intervals, forward KL is one seed, the dashed line is the teacher's share, horizontal bars are 95\% intervals of the starting share for the two starts that bracket the switch, which with the clean student are measured on 300,600 samples and the others on 14,400, hollow markers are the start not read, and diamonds are the second pair, Qwen3.5-4B on Fashion, whose teacher is at 22.0\% and whose starts are measured on 13,776 samples (\cref{app:srate}).
In \cref{fig:when}b--c each point is the seed mean of one checkpoint, three seeds on ESCI and five on Spider, and the triangles joined by lines are the runs from the weaker start at steps 42 to 210, from their common start (hollow) (\cref{app:quality}).
In \cref{fig:greedy} each point is one seed, and greedy validity falls once extrapolation pushes the share past 35\%.
In \cref{fig:math} lines are seed means and bars the range over one to three seeds (\cref{tab:mathcap}).
In \cref{fig:template}, (a) counts runs with greedy validity below 50\% with exact 95\% intervals and (b) gives the bracketless training rollouts per seed.

\end{document}